\documentclass{article}

\usepackage[preprint]{neurips_2021}

\usepackage{times}

\usepackage{soul}
\usepackage{url}
\usepackage[hidelinks]{hyperref}
\usepackage[utf8]{inputenc}
\usepackage[small]{caption}
\usepackage{graphicx}
\usepackage{amsmath}
\usepackage{booktabs}
\usepackage{booktabs}
\usepackage{amsthm}
\theoremstyle{plain}
\usepackage{graphics}
\usepackage{epstopdf}
\usepackage{lineno}
\usepackage{array}
\usepackage{amsmath,amssymb}
\usepackage{mathrsfs}
\usepackage{subfigure}
\usepackage{multirow}
\usepackage{xcolor}
\usepackage{bm}
\usepackage{eqparbox}
\usepackage{algorithm}
\usepackage{algorithmic}
\newtheorem{myTheo}{Theorem}
\newtheorem{myLemma}{Lemma}

\usepackage{amsfonts}
\usepackage{enumerate}
\usepackage{indentfirst}
\usepackage{wrapfig}

\newtheorem{myDef}{Definition}

\title{Non-Gradient Manifold Neural Network}

\author{%
	Rui Zhang
	\thanks{ Rui Zhang, Ziheng Jiao, Hongyuan Zhang, and Xuelong Li are with School of Computer Science and School of Artificial Intelligence, Optics and Electronics (iOPEN), Northwestern Polytechnical University, Xi'an 710072, Shaanxi, P. R. China.} \\
	Northwestern Polytechnical University\\
	\texttt{ruizhang8633@gmail.com} \\
	\And
	Ziheng Jiao \\
	Northwestern Polytechnical University\\
	\texttt{jzh9830@163.com} \\
	\And
	Hongyuan Zhang \\
	Northwestern Polytechnical University\\
	\texttt{hyzhang98@gmail.com} \\
	\And
	Xuelong Li \thanks{Xuelong Li is the Corresponding Author.}\\
	Northwestern Polytechnical University\\
	\texttt{xuelong\_li@ieee.org} \\
}

\begin{document}
	
	\maketitle
	\author{
		Rui Zhang\footnote\and
		Ziheng Jiao\and
		Hongyuan Zhang\And
		Xuelong Li{Contact Author}\\
		\affiliations
		School of Computer Science and School of Artificial Intelligence, Optics and Electronics (iOPEN), Northwestern Polytechnical University\\
		\emails
		ruizhang8633@gmail.com,
		jzh9830@163.com,
		hyzhang98@gmail.com,
		xuelong\_li@ieee.org
	}

	\begin{abstract}
		Deep neural network (DNN) generally takes thousands of iterations to optimize via gradient descent and thus has a slow convergence. In addition, softmax, as a decision layer, may ignore the distribution information of the data during classification. Aiming to tackle the referred problems, we propose a novel manifold neural network based on non-gradient optimization, i.e., the closed-form solutions. Considering that the activation function is generally invertible, we reconstruct the network via forward ridge regression and low rank backward approximation, which achieve the rapid convergence. Moreover, by unifying the flexible Stiefel manifold and adaptive support vector machine, we devise the novel decision layer which efficiently fits the manifold structure of the data and label information. Consequently, a jointly non-gradient optimization method is designed to generate the network with closed-form results. Eventually, extensive experiments validate the superior performance of the model.
	\end{abstract}
	
	\section{Introduction}
	Deep neural networks, a classic method in machine learning, consists of the number of neurons. Generally, it has three types of layers including the input layer, latent layer, and decision layer. In the DNN, any neuron of the $i$-th layer must be connected with any neuron of the $i$+$1$-th layer. Besides, the activation function is introduced to provide the model with powerful nonlinear mapping ability. Based on the simple operation of feature extraction and the good ability of feature representation, DNN has made great achievements in image processing \cite{NIPS2016_e56b06c5}, prediction systems \cite{1997Effective}, pattern recognition \cite{NIPS2016_6ea2ef73}, and so on.  
	
	Although DNN has many virtues, it is inefficient to extract meaningful features in high-dimensional data space. Additionally, owing to containing a large number of parameters, DNN is puzzled by over-fitting and low training speed. \cite{2004Cooperative} combines some generalized multi-layer perceptrons and utilizes the cooperative convolution to train the model. Apart from that, the backward propagation based on gradient descent is used to optimize the deep neural network. It propagates the residual error of the output layer to each neuron via chain rule and optimizes the parameters iteratively. Christopher \emph{et al}. \cite{de2015global} devise a step size scheme for SGD and prove that the proposed model can converge globally under broad sampling conditions. Aiming to accelerate the speed of the gradient descent, a reparameterization of the weight vectors is introduced into a neural network \cite{NIPS2016_ed265bc9}. TinyScript \cite{fu2020don} equips the activations and gradients with a non-uniform quantization algorithm to minimize the quantization variance and accelerate the convergence.

	The above-mentioned methods have made some improvements to the performance. However, they exist some weakness. On the one hand, gradient descent and its variants are commonly utilized to optimize these methods, which may lead to slow convergence and make the objective value unstable near the minima. On the other hand, due to simple calculation, softmax is extensively acted as the decision layer in DNN. However, it ignores the inner distribution of the data and lacks some interpretability.
	
	\textbf{Interestingly, the activation function is invertible in general}. Based on this, a novel manifold neural network is proposed in this paper. It not only can be optimized with a non-gradient strategy to accelerate the convergence but also efficiently fits the manifold structure and label information via a devised decision layer. In sum, our major contributions are listed as follows:
	
	$1)$ By utilizing the invertibility of the activation function, we reconstruct the neural network via the forward ridge regression and low rank backward approximation. Furthermore, a non-gradient strategy is designed to optimize the proposed network and can provide the network with analytic solutions directly.
	
	$2)$ By unifying the flexible Stiefel manifold and adaptive support vector machine reasonably, a novel decision layer is put forward. Compared with softmax, it has more interpretability via fusing the label information and distribution of the data. Moreover, this layer can be directly solved with an analytical method.
	
	$3)$ A novel manifold neural network is proposed and jointly optimized via a non-gradient algorithm. Moreover, extensive experiments verify the superiority and efficiency of the proposed model.

	\paragraph{Notations} In this paper, we use the boldface capital letters and boldface lower letters to represent matrics and vectors respectively. The $\bm{m_i}$ is supposed as the $i$-th coloumn of the matrix $\bm{M} \in R^{d \times c}$. $m_{ij}$ is the element of the $\bm M$. And $\bm{1}_c$ is a unit column vector with dimension c. ${\rm Tr}(\bm M)$ is the trace of $\bm M$. The $l_2$-norm of vector $\bm{v}$ and the Frobenius norm of $\bm{M}$ is denoted as $\|\bm v\|_2$ and $\|\bm M\|_F$. The Hadamard product is defined as $\odot$. $(\cdot)_+$ is $\max(\cdot, 0)$.

	\section{Related Work}
	
	\subsection{Deep Neural Network with Softmax} \label{dnn problem}
	As is known to us, DNN is a basic method in machine learning \cite{liu2017survey}. Owing to the excellent ability in representation learning, it has been widely used in many fields such as classification. Generally, DNN is divided into two stages, forward propagation, and backpropagation. Given a dataset $\bm X=\left\{ \bm{x_1},\bm{x_2},...,\bm{x_n} \right\} \in R^{d \times n}$ and $\bm Y \in R^{c \times n}$. Among them, $\bm X$ denotes the feature matrix with $n$ samples and $d$ dimensions. $\bm Y$ is a one-hot label matrix and $c$ is the class of the data. Based on this, the forward propagation of can be defined as 
	\begin{equation}
		\label{Neural Network}
		\left\{ 
		\begin{array}{l}
			\bm{Z}^l = (\bm{W}^l)^T \bm{H}^{l-1} + \bm{b}^l\bm{1}_n^T \\
			\bm{H}^l = \sigma(\bm{Z}^l)
		\end{array} 
		\right.
		,
	\end{equation}
	where $l \in \left\{ 0,1,...,t \right\}$ is the $l$-th layer in the DNN, $\bm{W}^l \in R^{d_l \times l_{l-1}}$ is the weight matrix in the $l$-th layer and $\bm{b}^l \in R^{d_l \times 1}$ is the bias vector. The input feature and output feature of the $l$-th layer are individually defined as $\bm{H}^{l-1} \in R^{d_{l-1} \times n}$ and $\bm{H}^l \in R^{d_l \times n}$ where $\bm{H}^{0}$ is equivalent to $\bm X$ and $\bm{H}^{t}$ is viewed as a predict label $\bm{\hat{Y}}$. $\bm{Z}^l \in R^{d_l \times n}$  is the latent feature matrix. $\sigma (\cdot)$ is the activation function. For classification, the softmax is employed to act as the decision layer and formulated as
	\begin{equation}
		\label{softmax}
		\hat{y}_{ij} = \frac{e^{z^{t}_{ij}}}{\sum_{j=1}^{c} e^{z^{t}_{ij}}}.
	\end{equation}
	Cross-entropy is widely utilized as the loss function. Given the learning rate $\alpha$, the parameters $\theta$ of the model can be optimized with the gradient descent via 
	\begin{equation}
		\label{gradient descent}
		\theta:=\theta-\alpha \frac{\partial Loss}{\partial \theta}.
	\end{equation}

\begin{figure*}[t]
	\centering
	\includegraphics[width=140mm]{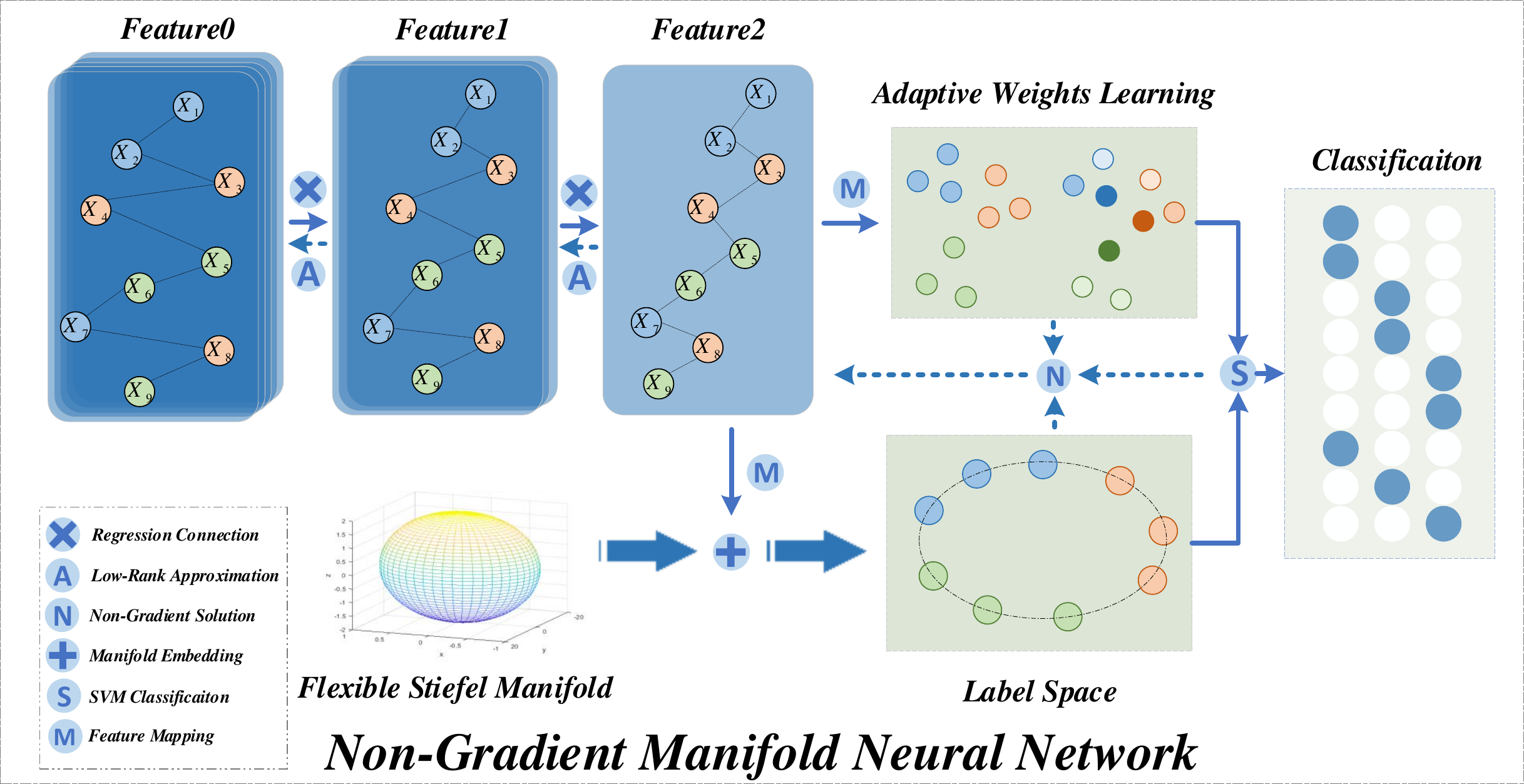}\\
	\caption{A framework of the proposed model. The blue, orange and green nodes represent three different data. As is shown, the reconstruction regression network extracts the deep features whose distribution satisfies a latent regression. Then, on the one hand, a flexible Stiefel manifold is embedding into the label space. On the other hand, these features are assigned with different weights via adaptive learning. Finally, unifying the two kinds of information, the model predicts the labels via SVM.}
	\label{framework}
\end{figure*}
	
	From Eq. (\ref{gradient descent}), DNN generally takes thousands of iterations to optimize and obtain the approximate solutions of $\theta$. Besides, the value of the learning rate extremely affects the performance and convergence of the model. For example, if the value of $\alpha$ is improper, the loss will be unstable near the minima and the model can not converge. Therefore, we reconstruct the network via ridge regression to accelerate the convergence.  In addition, as shown in Eq. (\ref{softmax}), softmax mainly classifies the data according to the value of deep features and the property of the exponential function. However, it is weak to utilize the inner distribution of the data and exists AI ethical  problems with low interpretability. Compared with softmax, the SVM aims to find the hyper-planes in label space for classification.
	%	However, it lacks the capability to mine the latent manifold structure of the data. Compared with it, the multi-class SVM (Eq.(\ref{OvR SVM})) maps the nodes into label space based on their feature value and finds the optimal hyper-planes for classification.
	
	\subsection{Multi-Class Support Vector Machine}
	As a fundamental model of machine learning, Support Vector Machine (SVM) \cite{pradhan2012support} has been developed mostly in classification and regression. It utilizes the topological information in the label space to find a hyper-plane for binary classification and make the margin maximum. The classic objective function of SVM can be defined as 
	\begin{equation}
		\label{binary svm}
		\min \limits_{\bm{w}} \frac{1}{2} \|\bm{w}\|_2^2, \quad s.t.\ y_i(\bm{w}^T \bm{x}_i+b)\geq 1,\ i=1,...,n,
	\end{equation}
	where $\bm{w}$ and $b$ are the parameters to learn, $\bm{x}_i \in R^d$ is the $i$-th sample, and $y_i \in \left\{ -1, 1\right\}$ is the label of $\bm{x}_i$. Furthermore, define the $\xi(\bm{w},b|\bm{x}_i)$ as the misclassification loss, Eq.        (\ref{binary svm}) can be written as 
	${\min}_{\bm{w},b} \sum_{i=1}^{n}\xi(\bm{w},b|\bm{x}_i)+\lambda ||w||_2^2$ where $\lambda$ is the trade-off parameters. Apart from that, the soft margin generally is utilized to avoid the over-fitting and the $\xi(\bm{w},b|\bm{x}_i)$ represnets the square hinge loss which can be defined as $\xi(\bm{w},b|\bm{x}_i)=(1-y_i(\bm{w}^T \bm{x}_i+b))_+^2$. Thus, the primal loss can be given as
	\begin{equation}
		\label{hinge svm loss}
		\min \limits_{\bm{w},b} \sum_{i=1}^{n} (1-y_i(\bm{w}^T \bm{x}_i+b))_+^2 + \lambda ||\bm{w}||_2^2.
	\end{equation}
	
	The One-versus-Rest (OvR) strategy is employed to extend the binary SVM into multi-class scenarios. Besides, all classifiers can be simultaneously optimized like \cite{Crammer2002On}. Therefore, the multi-class formulation of  Eq. (\ref{hinge svm loss}) can be written as
	\begin{equation}
		\label{OvR SVM}
		\min \limits_{\bm{w}_j,b_j} \sum_{j=1}^{c} \sum_{i=1}^{n} [(\bm{w}_j^T \bm{x}_i+b_j)-y_{ij}-y_{ij} m_{ij}]^2 + \lambda ||\bm{w}_j||_2^2,
		%	\Leftrightarrow
		%	\min \limits_{\bm{W},\bm{b},\bm{M}\geq0} \sum_{i=1}^{n} ||(\bm{W}^T\bm{x}_i+\bm{b})-\bm{y}_i-\bm{y}_i \odot \bm{m}_i||_2^2+\lambda||\bm{W}||_F^2
	\end{equation}
	where $c$ is the number of classes and $y_{ij}$ is the element of the one-hot label matrix $\bm{Y} \in R^{c \times n}$. $m_{ij}$ is the element of slack matrix $\bm{M}\in R^{c \times n}$. 
	Apart from that, \cite{2012Multi} introduces the $l_{2,1}$-norm to avoid the model misled by the outliers. Aiming to improve the performance, the scaling factor is integrated with SVM for feature selection \cite{2017Feature}. Although the method mentioned above has improved the performance of the SVM, they ignore the latent distribution of the data. In this paper, we not only unify the flexible Stiefel manifold with adaptive SVM to fit the inner manifold structure of the deep embeddings but also employ it as a novel decision layer with higher interpretability for a manifold neural network.

	\section{Non-Gradient Manifold Neural Network}
	Aiming to tackle the weakness referred to in Subsection \ref{dnn problem}, we propose a novel manifold neural network. The network is reconstructed via ridge regression and low rank approximation. Apart from that, it embeds a flexible Stiefel manifold into an adaptive support vector machine to act as a decision layer. More importantly, a novel backward optimization with non-gradient is designed to solve the model and obtain the closed-form results directly. The framework is illustrated in Fig \ref{framework}. Besides, the proof of all lemmas and theorems is listed in supplementary.
	
	\subsection{Reconstruct Network via Ridge Regression} 
	Aiming to accelerate the convergence and obtain the closed-form results, we reconstruct the neural network. The forward propagation Eq. (\ref{Neural Network}) can be reformulated as
	\begin{equation}
		\label{regression forward}
		\bm{H}^l=\sigma((\bm{W}^l)^T \bm{H}^{l-1} + \bm{b}^l\bm{1}_n^T) \Leftrightarrow
		\min \limits_{\bm{H}^l} ||\bm{H}^l - \sigma((\bm{W}^l)^T \bm{H}^{l-1} + \bm{b}^l\bm{1}_n^T)||_F^2,
	\end{equation}
	where the left part can be obtained by taking the derivation of the right equation w.r.t $\bm{H}^l$.
	
	\textbf{Assumption:} Noticing that the activation function is generally invertible, we can define the invertibility of the activation  function as $\sigma^{-1}(\cdot)$.
	
	\textbf{Forward ridge regression:} 	
	From Eq. (\ref{regression forward}), this forward propagation tries to embed more original and real data information in the deep reconstruction features. Moreover, based on the above assumption, we can make some relaxation  and the backward optimization of each layer can be defined as
	\begin{equation}
		\label{regression backward}
		\min \limits_{\bm{W}^l, \bm{b}^l} ||(\bm{W}^l)^T\bm{H}^{l-1}+\bm{b}^l\bm{1}_n^T-\sigma^{-1}(\bm{H}^l)||_F^2.
	\end{equation} 
	Aiming to work out overfitting and has excellent performance on a small dataset, we introduce the $F$-norm regularization for the designed network:
	\begin{equation}
		\label{reg_regular}
		\min \limits_{\bm{W}^l, \bm{b}^l} ||(\bm{W}^l)^T\bm{H}^{l-1}+\bm{b}^l\bm{1}_n^T-\sigma^{-1}(\bm{H}^l)||_F^2 + \lambda ||\bm{W}^l||_F^2,
	\end{equation}
	where $\lambda$ is the trade-off of the regularization term. The $\bm{b}^l$ can be solved by Lemma \ref{cal_b}.
	
	\begin{myLemma} 
		\label{cal_b}
		Given a feature matrix $\bm X$ and a label matrix $\bm Y$, the problem 
		$\min \limits_{\bm{b}} ||\bm{W}^T\bm{X} + \bm{b}\bm{1}_n^T-\bm{Y}||_F^2+\lambda ||\bm{W}||_F^2$
		can be solved by $\bm{b}=\frac{\bm{Y}-\bm{W}^T\bm{X}}{n} \bm{1}_n$.
		%	, where $\bm{S}=\bm{E}\bm F^T$ and $[\bm E, \sim, \bm{F}^T]=svd(\bm X \bm{R})$. $\bm R$ is defined as $\bm{R}=[\bm{Y}, \bm{X}\bm{V}]$ and $\bm V \in R^{d \times (d-c)}$ is generated in the orthogonal complement spaces of $\bm U$.
	\end{myLemma}
	
	Having solved the $\bm b$, we replace it into the Eq. (\ref{reg_regular}) like
	\begin{equation}
		\label{simplify with bias}
		\begin{split}
			\min \limits_{\bm{W}^l} &||\bm{W}^T\bm{X} + \frac{\sigma^{-1}(\bm{H}^l)-\bm{W}^T\bm{X}}{n} \bm{1}_n \bm{1}_n^T-\sigma^{-1}(\bm{H}^l)||_F^2 + \lambda ||\bm{W}^l||_F^2\\
			=\min \limits_{\bm{W}^l} &||\bm{W}^T\bm{X}\bm{C} -\sigma^{-1}(\bm{H}^l)\bm{C}||_F^2 + \lambda ||\bm{W}^l||_F^2,\\
		\end{split}
	\end{equation}
	where $\bm{C}=\bm{I}_n-\frac{1}{n}\bm{1}_n\bm{1}_n^T$ and $\bm{I}_n \in R^{n \times n}$ is a identity matrix. $\bm{C} \in R^{n \times n}$ is a centralized matrix according to the column. To represent more conveniently, we define the matrix $\hat{\bm{Y}}^l=\sigma^{-1}(\bm{H}^l)$ and the problem Eq. (\ref{simplify with bias}) can be transformed into 
	\begin{equation}
		\label{centralized loss}
		\min \limits_{\bm{W}^l} ||(\bm{W}^l)^T\bm{X}_{\bm C}^l- \bm{\hat{Y}}_{\bm C}^l||_F^2+\lambda ||\bm{W}^l||_F^2,
	\end{equation}
	where $\bm{X}_{\bm C}^l =\bm{X}^l \bm{C}$ and $\bm{\hat{Y}}_{\bm C}^l=\bm{\hat{Y}}^l \bm{C}$. Problem (\ref{centralized loss}) can be solved by Lemma \ref{cal_w}.
	
	\begin{myLemma} 
		\label{cal_w}
		Given a feature matrix $\bm X$ and a label matrix $\bm Y$, the problem 
		$\min \limits_{\bm{W}} ||\bm{W}^T\bm{X}-\bm{Y}||_F^2+\lambda ||\bm{W}||_F^2$
		can be solved by $\bm{W}= (\bm{X}\bm{X}^T+\lambda\bm{I}_d)^{-1} \bm{X}\bm{Y}^T$.
	\end{myLemma}
	
	By stacking the novel layer Eq. (\ref{regression forward}), the model successfully extracts the deep feature embedded with the original data information via ridge regression. Moreover, Lemma \ref{cal_b} and Lemma \ref{cal_w} are utilized to obtain the analytic solution of the parameters without gradient. 
	
	\textbf{Low rank backward approximation:} Aiming to propagate the label information from the output layer to the input layer, the low rank reconstruction is employed during the backward optimization. Taking the $l$-th layer as an example,  the problem can be reformulated as
	\begin{equation}
		\label{label backward}
		\min \limits_{\bm{H}^{l-1}} ||(\bm{W}^l)^T\bm{H}^{l-1}+\bm{b}^l\bm{1}_n^T-\tilde{\bm{Y}}^l||_F^2,
	\end{equation}
	where $\tilde{\bm{Y}}^l$ is the label matrix in $l$-th layer. Since the analytical results of parameters can be directly obtained via the above lemmas, Eq. (\ref{label backward}) is equal to
	% ${\bm{W}^l}^T\bm{H}^{l-1} = \tilde{\bm{Y}}^l - \bm{b}^l\bm{1}_n^T$.
	\begin{equation}
		\label{label transform}
		\bm{W}^l(\bm{W}^l)^T\bm{H}^{l-1} = \bm{W}^l (\tilde{\bm{Y}}^l - \bm{b}^l\bm{1}_n^T).
	\end{equation}
	The left of Eq. (\ref{label transform}) shows that the input feature firstly is mapped into a  $c$-dimensional low rank subspace and reconstructed into a $d$-dimensional space. This procedure can be viewed as 
	\begin{equation}
		\label{low rank reconstruct}
		\min \limits_{rank(\bm \bar{\bm H}) \leq c} ||\bm{H}^{l-1}-{\bar{\bm H}}||_F^2 
		=\min \limits_{\bm{V},\bm{W}^T\bm{W}=\bm{I}_c} ||\bm{H}^{l-1}-\bm{W}\bm{V}^T||_F^2, 
		%	\begin{split}
		%		\min \limits_{rank(\bm \dot{\bm H}) \leq c} &||\bm{H}^{l-1}-{\dot{\bm H}}||_F^2 \\
		%		=\min \limits_{\bm{V},\bm{W}^T\bm{W}=\bm{I}_c} &||\bm{H}^{l-1}-\bm{W}\bm{V}^T||_F^2, \\
		%	\end{split}
	\end{equation}
	where $\bar{\bm H}$ is the reconstructed matrix of the input feature. $\bm{W}\in R^{d\times c}$ is employed as a base vector matrix and $\bar{\bm H}$ can be decomposed as $\bar{\bm H}=\bm{W}\bm{V}^T$. Then, the matrix $\bm{V}$ is easily solved via taking the derivative of Eq. (\ref{low rank reconstruct}) w.r.t $\bm{V}$ and setting it to 0. The result is like $\bm{V}=\bm{X}^T\bm{W}$. Therefore,  $\bm{H}^{l-1}$ is an approximate solution of $\bm{W}^l{\bm{W}^l}^T\bm{H}^{l-1} $ based on low rank reconstruction and the label backward can be defined like
	\begin{equation}
		\label{label reconstruct}
		\bar{\bm H}^{l-1}=\bm{W}^l (\tilde{\bm{Y}}^l - \bm{b}^l\bm{1}_n^T).
	\end{equation}
	In order to promote the stability of the model, Softmax is introduced to map the reconstructed label matrix $\tilde{\bm{Y}}^{l-1}=\bar{\bm H}^{l-1}$ into $(0,1)$. The whole network works as Algorithm \ref{neural network}.
	
	\begin{algorithm}[t]
		\caption{Reconstruct Neural Network via Ridge Regression}
		\label{neural network}
		\begin{algorithmic}[1] 
			\REQUIRE data matrix $\bm{X}$, (i.e., $\bm{H}^0$), one-hot label matrix $\bm{Y}$.\\
			%		\ENSURE the $\bm{Y}_u$.\\
			\STATE Initialize $\bm{W}^{l}$ and $\bm{b}^{l}$ of each layer;
			\WHILE{$not \; Convergence$}
			\STATE Forward inference via Eq. (\ref{regression forward});
			\FOR{$l=t$ to $1$}
			\STATE Optimize the loss Eq. (\ref{reg_regular}) of $l$-th layer via Lemma \ref{cal_b} and Lemma \ref{cal_w};
			\STATE Propagate the label information to the $l$-$1$-th layer via Eq. (\ref{label reconstruct});
			\ENDFOR 
			\ENDWHILE
			%			\STATE Calculate the $\bm{Y}_u$ via $\bm{H}^{(m)}\bm{U}$
		\end{algorithmic} 
	\end{algorithm}

	\subsection{Decision Layer with Flexible Stiefel Manifold}
	Although softmax can predict the label according to the value of features and has been widely utilized as a decision layer for classification, it ignores the latent distribution of the data and is weak in interpretation. As is known to us, the support vector machine (SVM) utilizes the distribution of data in label space to find the optimal hyper-planes whose marges generated from the support vector are maximum. Therefore, the multi-class SVM is employed as the decision layer for classification. Inspired by \cite{9390382}, Eq. (\ref{OvR SVM}) can be reformulated as
	\begin{equation}
		\label{multi_class SVM}
		\min \limits_{\bm{W},\bm{b},\bm{M} \geq 0} \sum \limits_{i=1}^n \underbrace{||\bm{W}^T\bm{x}_i+\bm{b}-\bm{y}_i-\bm{y}_i \odot \bm{m}_i||_2^2}_{f_i}+\lambda \| \bm W\|_F^2,
	\end{equation}
	where $\bm{M}=[\bm{m}_1,\bm{m}_2,...,\bm{m}_n]\in R^{c \times n}$ is a slack variable to encode the loss of each $\bm{x}_i \in R^{d \times 1}$. In the label space, because some nodes may be contaminated by noises or far away from the hyper-lines, these $\bm{f}_i$ will be much large, which leads to the model mainly optimize this part of loss and disturbs the performance sharply. Therefore, the adaptive weight vector $\bm{\alpha}\in R^{n \times 1}$ is introduced into the decision layer to enable the model to pay more attention to good nodes for example support vector with $\bm{f}_i=0$. It can be formulated as
	\begin{equation}
		\label{adaptive weigt svm}
		\min \limits_{\bm{\alpha}^T \bm{1}_n=1,\bm{W},\bm{b},\bm{M} \geq 0} \sum \limits_{i=1}^n \alpha_i f_i +\lambda \|\bm W\|_F^2+\gamma ||\bm{\alpha}||_2^2,
	\end{equation}
	where $\gamma$ is a trade-off coefficient and can control the sparsity of $\bm{\alpha}$. We can obtain the closed-form results of this weight vector via Theorem \ref{cal_alpha}.
	
	\begin{myTheo} 
		\label{cal_alpha}
		Suppose that $f_{(1)} \leq f_{(2)} \leq ... \leq f_{(n)}$. If $\gamma=\frac{n-1}{2}f_{(n)}-\frac{1}{2} \sum_{i=1}^{n-1}f_{(i)}$, the optimal $\bm{\alpha}$ is 
		\begin{equation}
			\label{cal_a}
			\alpha_i = (\frac{f_{(n)} - f_{(n-1)}}{(n-1) f_{(n)} - \sum \limits_{j=1}^k f_{(j)}})_+ .
		\end{equation}
	\end{myTheo}
	
	Inspired by the successful application of manifold learning in representation learning, we employ the flexible Stiefel Manifold \cite{9134971} formulated in Definition \ref{stiefel} to explore the inner distribution of the data. 
	\begin{myDef}(Flexible Stiefel Manifold)
		\label{stiefel}
		For a matrix $\bm{A} \in R^{a \times b}$, $\bm A$ obeys flexible Stiefel manifold if $\bm A \in \left\{ \bm{B} \in R^{a \times b} | \bm{B}\bm{D}\bm{B}^T+\bm{E}=\bm{I}_a, \bm{D} \geq 0 \right\}$, where $\bm E$ is a residual matrix.
	\end{myDef}
	
	According to this definition, we assume that the predicted label vectors exist in a flexible Stiefel manifold space. To keep the discussion convenient, we introduce the weighted-centralized matrix $\hat{\bm{C}}=\bm{I}_n-\frac{1}{\bm{1}_n^T \bm{D} \bm{1}_n}\bm{D}\bm{1}_n\bm{1}_n^T$ and $\bm{D}={\rm diag}(\bm{\alpha})$. Therefore, the label manifold space can be defined as
	\begin{equation}
		\label{flexible manifold}
		\bm{W}^T\bm{X}\hat{\bm{C}}\bm{D}\hat{\bm{C}}^T\bm{X}^T\bm{W}+\bm{E}=\bm{I}_c.
	\end{equation}
	Owing to $\bm{E}$, this flexible manifold can fit plenty of latent irregular manifold structures. Generally, the error matrix is assumed to satisfy $\bm{E}=\hat{\lambda} \bm{W}^T \bm{W}$. Then, the energy can be defined with $\| \bm E \|=\hat{\lambda}tr(\bm{W}^T \bm{W})=\hat{\lambda}\|\bm W\|_F^2$. To avoid introducing new parameters, we utilize the $\lambda$ to substitute for $\lambda+\hat{\lambda}$ with a hyper-parameter trick. Therefore, the problem can be reformulated as
	%	\begin{equation}
	%		\label{decision function}
	%		\min \limits_{\bm{\alpha} \geq 0, \bm{\alpha}^T \bm{1}_n=1, \bm{M} \geq 0,\bm{W},\bm{b}} \sum \limits_{i=1}^n \alpha_i f_i +\lambda \|\bm W\|_F^2 +\gamma ||\bm{\alpha}||_2^2, \  
	%		s.t. \ \bm{W}^T(\bm{X}\hat{\bm{C}}\bm{D}\hat{\bm{C}}^T\bm{X}^T+\lambda \bm{I}_d)\bm{W}=\bm{I}_c\\
	%		.
	%	\end{equation}	
	\begin{equation}
		\label{decision function}
		%		\begin{array}{l}
		\begin{split}
			&\min \limits_{\bm{\alpha}, \bm{M},\bm{W},\bm{b}} \sum \limits_{i=1}^n \alpha_i ||(\bm{W}^t)^T\bm{h}_i^{t-1}+\bm{b}-\bm{y}_i-\bm{y}_i \odot \bm{m}_i||_2^2 +\lambda \|\bm W\|_F^2 +\gamma ||\bm{\alpha}||_2^2\\
			&~~~~~~~~s.t.\  \bm{W}^T(\bm{X}\hat{\bm{C}}\bm{D}\hat{\bm{C}}^T\bm{X}^T+\lambda \bm{I}_d)\bm{W}=\bm{I}_c,\bm{\alpha}^T \bm{1}_n=1, \bm{\alpha} \geq 0, \bm{M} \geq 0.
		\end{split}
		%		\end{array}
		%		\min \limits_{\bm{\alpha} \geq 0, \bm{\alpha}^T \bm{1}_n=1, \bm{M} \geq 0,\bm{W} \in \mathcal{F},\bm{b}} \sum \limits_{i=1}^n \alpha_i f_i +\lambda \|\bm W\|_F^2 +\gamma ||\bm{\alpha}||_2^2,
	\end{equation}
	
	%	where $\mathcal{F}=\left\{ \bm W | \bm{W}^T(\bm{X}\hat{\bm{C}}\bm{D}\hat{\bm{C}}^T\bm{X}^T+\lambda \bm{I}_d)\bm{W}=\bm{I}_c  \right\}$ is a flexible Stiefel manifold.
	By embedded the flexible manifold, Eq. (\ref{decision function}) successfully utilizes the inner distribution of the data for classification. The following Theorem \ref{solve manifold} can provide the closed-form results for this problem.
	
	\begin{myTheo}
		\label{solve manifold}
		Suppose that $\bm{\alpha}$ is a constant and let $\bm{G}=\bm{Y}-\bm{Y}\odot \bm{M}$. $\bm W$ and $\bm b$ can be solved by
		\begin{equation}
			\label{W_b}
			\bm{W} = \bm{S}^{-1} \bm{U} \bm{\Lambda}^T \bm{V}^T \quad and \quad
			\bm{b} = \frac{\bm{G}\bm{D}\bm{1}_n-\bm{W}^T\bm{X}\bm{D}\bm{1}_n}{\bm{1}_n^T\bm{D}\bm{1}_n}
			%			\min \limits_{\bm U} ||\bm{H}^{(m)}_{n_l} \bm{U} - \bm{Y}_{n_l}||_F^2 \\
			%		\left\{
			%		\begin{array}{l}
			%			\bm{W} = \bm{S}^{-1} \bm{U} \bm{\Lambda}^T \bm{V}^T \\
			%			\bm{b} = \frac{\bm{G}\bm{D}\bm{1}_n-\bm{W}^T\bm{X}\bm{D}\bm{1}_n}{\bm{1}_n^T\bm{D}\bm{1}_n}\\
			%			%			\min \limits_{\bm U} ||\bm{H}^{(m)}_{n_l} \bm{U} - \bm{Y}_{n_l}||_F^2 \\
			%		\end{array}
			%		\right.
			%		,
		\end{equation}
		where $\bm{S}=(\bm{X}\hat{\bm{D}}\bm{X}^T+\lambda \bm{I}_d)^{\frac{1}{2}}$, $\hat{\bm{D}}=\hat{\bm{C}}\bm{D}\hat{\bm{C}}^T$ and $\bm{\Lambda}=[\bm{I}_c, \bm{0}]\in R^{c\times d}$. $\bm{U},\bm{V}$ are the left and right singular of $\bm{S}^{-1}\bm{X}\hat{\bm{D}}\bm{G}$, respectively. Having obtained the solution of $\bm W, \bm b$, slack variable $\bm M$ can be solved by
		\begin{equation}
			\label{cal_M}
			\bm{m}_i=(\bm{y}_i \odot (\bm{W}^T \bm{x}_i+\bm{b})-\bm{1}_c)_+.
		\end{equation}
	\end{myTheo}
	
	\begin{algorithm}
		\caption{Manifold Neural Network}
		\label{Whole Network}
		\begin{algorithmic}[1] 
			\REQUIRE data matrix $\bm{X}$, (i.e., $\bm{H}^0$), one-hot label matrix $\bm{Y}$, trade-pff parameters $\lambda$ and $\gamma$.\\
			%		\ENSURE the $\bm{Y}_u$.\\
			\STATE Initialize $\bm{W}^{l}$ and $\bm{b}^{l}$ of each layer;
			\WHILE{$not \; Convergence$}
			\STATE Forward inference via Eq. (\ref{manifold network});
			\STATE Optimize the $\bm \alpha$, $\bm{M}$, $\bm{W}^t$ and $\bm{b}^t$ via Theorem \ref{cal_alpha} and Theorem \ref{solve manifold};
			%		\STATE Update $\bm{M}$ by Eq. (\ref{cal_M}).
			%		\STATE Update $\bm \alpha$ by Eq. (\ref{cal_a}).
			%		\STATE Update $\bm{W}^t$ and $\bm{b}^t$ by Eq. (\ref{W_b}).
			\FOR{$l=L-1$ to $1$}
			\STATE Optimize the loss Eq. (\ref{reg_regular}) of $l$-th layer via Lemma \ref{cal_b} and Lemma \ref{cal_w};
			\STATE Propagate the label information to the $l-1$-th layer via Eq. (\ref{label reconstruct});
			\ENDFOR 
			\ENDWHILE
			%			\STATE Calculate the $\bm{Y}_u$ via $\bm{H}^{(m)}\bm{U}$
		\end{algorithmic} 	
	\end{algorithm}
	
	\subsection{A Novel Manifold Neural Network}
	Although Eq. (\ref{decision function}) not only has strong interpretability but also explores the latent distribution of the data with flexible Stiefel manifold, it can not has good performance when facing non-linear problems. Aiming to solve this problem, we use the reconstructed neural network proposed in Eq.  (\ref{regression forward}) to extract the feature and employ  Eq.  (\ref{decision function}) to learn the latent manifold of the embedding for classification. Then, we finally put forward a novel manifold neural network like
	\begin{equation}
		\label{manifold network}
		\left\{
		\begin{array}{l}
			\begin{split}
				&\min \limits_{\bm{H}^l} ||\bm{H}^l - \sigma((\bm{W}^l)^T \bm{H}^{l-1} + \bm{b}^l\bm{1}_n^T)||_F^2 \\
				&\min \limits_{\mathcal{F}} \sum \limits_{i=1}^n 
				\alpha_i {||(\bm{W}^t)^T\bm{h}_i^{t-1}+\bm{b}-\bm{y}_i-\bm{y}_i \odot \bm{m}_i||_2^2}
				+\lambda \|\bm{W}^t\|_F^2 +\gamma ||\bm{\alpha}||_2^2, \\
				%				&~~~~~~~~~~~~~~~~~s.t.\  	\bm{W}^t \in \mathcal{F}, \bm{\alpha}^T \bm{1}_n=1, \bm{\alpha} \geq 0, \bm{M} \geq 0,
			\end{split}
			
			%		\begin{split}
			%			\min \limits_{\bm{\alpha},\bm{W}^t,\bm{b}^t,\bm{M}} &\sum \limits_{i=1}^n 
			%			\alpha_i \underbrace{||\bm{W}^T\bm{x}_i+\bm{b}-\bm{y}_i-\bm{y}_i \odot \bm{m}_i||_2^2}_{f_i}
			%			+\lambda \|\bm{W}^t\|_F^2 +\gamma ||\bm{\alpha}||_2^2, \\
			%			s.t.& {\bm{W}^t}^T(\bm{H}^t\hat{\bm{C}}\bm{D}\hat{\bm{C}}^T{\bm{H}^t}^T+\lambda \bm{I}_d)\bm{W}^t=\bm{I}_c, \\
			%			& \bm{\alpha}^T \bm{1}_n=1, \bm{\alpha} \geq 0, \bm{M} \geq 0.
			%		\end{split}
		\end{array}
		\right.
	\end{equation}
	where $\mathcal{F}=\left\{ \bm{W}^t,\bm{b}^t,\bm{\alpha},\bm{M} \right\}$ satisfies the same  constraints shown in Eq. (\ref{decision function}) , $l \in \left\{ 0,1,...,t \right\}$ is the $l$-th layer in manifold neural network and $\bm{y}_i \in R^{c}$ is the one-hot label of the $i$-th node. Moreover, the whole network is jointly optimized via Algorithm \ref{Whole Network}.

	\noindent \textbf{The Merits of the Proposed Model: } 
	Compared with the classic deep neural network, the proposed network obtains closed-form results via ridge regression reconstruction and accelerates convergence. Besides, by embedding the flexible Stiefel manifold into the decision layer, the model successfully learns the distribution of the deep embeddings. Moreover, the multi-class SVM with adaptive weights is unified with the decision layer, which makes the model more explicable. Finally, we propose a novel manifold neural network and design a joint optimization strategy that can directly obtain the closed-form solution of each parameter.

	\subsection{Time Complexity}
	Owing to reconstructing the neural network via ridge regression, the computational complexity is $O(n^2d_l)$. For the decision layer, calculating the $\bm M$ and $\bm \alpha$ need $O(ncd_{t-1})$ and $O(n \log n+n)$. To obtain the $\bm{W}^t$, the total time complexity is $O(nd_{t-1}c+d_{t-1}^2c+d_{t-1}c^2+d_{t-1}^3)$. Then, update the $\bm{b}^t$ takes $O(cn+n+d_{t-1}n+cd_{t-1})$. In sum, the whole model needs $O(T(tn^2d_l+ncd_{t-1}+n \log n+d_{t-1}^2c+d_{t-1}c^2))$, where $t$ is the number of the feature extraction layers and $T$ represents the number of iterations. In practice, since  $n$ is much larger than $d_l$ and $c$, the time complexity can be approximated as $O(Ttn^2)$.
	
	\begin{table*}[bp]
		\renewcommand\arraystretch{1.3}
		\centering
		\caption{Datasets Description}
		\label{table_datasets}
		\scalebox{0.9}{
			\begin{tabular}{ccccccc}
				\toprule
				\textbf{Dataset}       & \textbf{AT\&T} & \textbf{UMIST} & \textbf{WAVEFORM} & \textbf{MNIST-Mini} & \textbf{MNIST} & \textbf{FashionMNIST} \\ \hline
				\textbf{\# of samples} & 400            & 575            & 2746              & 10000               & 70000          & 70000                 \\
				\textbf{Features}      & 1024           & 1024           & 21                & 784                 & 784            & 784                   \\
				\textbf{Classes}       & 40             & 20             & 3                 & 10                  & 10             & 10                    \\ 
				\bottomrule
			\end{tabular}
		}
	\end{table*}
	
	\section{Experiment}
	In this paper, we propose a novel neural network and a non-gradient optimization strategy. Therefore, the experiment mainly verifies the feasibility and performance of the proposed network based on the completely non-gradient strategies. Therefore, we select six comparison methods including DNN. 
	
	\begin{table*}[]
		\renewcommand\arraystretch{1.3}
		\centering
		\caption{Accuracy(\%)  and F1-score(\%) on Benchmark Datasets}
		\label{table_result}
		\scalebox{0.80}{
			\begin{tabular}{cccccccc}
				\toprule
				\multicolumn{2}{c}{}	& \textbf{AT\&T}& \textbf{WAVEFORM}	& \textbf{UMIST}& \textbf{MNIST-MINI}& \textbf{MNIST}	& \textbf{FashionMNIST}\\ \hline
				& \emph{Acc}			& 8.77$\pm$0.41 & 56.98$\pm$0.19	& 20.00$\pm$1.19& 23.58$\pm$0.02	& 23.58$\pm$0.15	& 36.86$\pm$0.12\\
				\multirow{-2}{*}{\textbf{RidgeReg}} 
				& \emph{F1}				& 6.94$\pm$0.22	& 58.11$\pm$0.19	& 14.98$\pm$2.00& 13.81$\pm$1.74	& 16.99$\pm$0.50	& 24.69$\pm$0.07\\
				
				& \emph{Acc}			& 5.83$\pm$1.78	& 57.04$\pm$0.78	& 21.39$\pm$2.44& 23.53$\pm$0.26	& 23.29$\pm$0.22	& 36.85$\pm$0.07\\
				\multirow{-2}{*}{\textbf{LassoReg}} 
				& \emph{F1}				& 4.40$\pm$1.41	& 58.22$\pm$0.81	& 14.77$\pm$1.83& 17.21$\pm$0.51	& 17.75$\pm$1.19	& 25.32$\pm$0.86\\
				
				& \emph{Acc}			& 86.67$\pm$1.93& 72.60$\pm$0.53	& 75.14$\pm$1.83& \underline{81.29$\pm$0.22}	& \textbf{88.71$\pm$0.12} & 79.49$\pm$0.25\\
				\multirow{-2}{*}{\textbf{SmartSVM}}           
				& \emph{F1}				& 87.05$\pm$1.17& 72.38$\pm$0.55	& 72.52$\pm$1.59& \underline{81.05$\pm$0.23}	& \textbf{88.60$\pm$0.12} & 79.49$\pm$0.25\\
				
				& \emph{Acc}			& \underline{90.00$\pm$2.93}& 76.33$\pm$1.23	& \underline{96.88$\pm$0.62}& 75.25$\pm$1.67	& 62.15$\pm$0.11	& 43.46$\pm$0.17       \\
				\multirow{-2}{*}{\textbf{MKVM}}               
				& \emph{F1}				& \underline{92.92$\pm$2.86}& 72.24$\pm$1.64	& \underline{95.84$\pm$0.58}& 74.47$\pm$1.89	& 60.95$\pm$0.12	& 42.18$\pm$0.16       \\
				
				& \emph{Acc}			& 51.67$\pm$8.07& 67.48$\pm$6.41	& 89.02$\pm$0.05& 75.47$\pm$2.69	& 55.34$\pm$0.05	& 50.94$\pm$2.39\\
				\multirow{-2}{*}{\textbf{SVM}}               
				& \emph{F1}				& 56.59$\pm$8.05& 56.59$\pm$10.99	& 87.33$\pm$0.16& 75.56$\pm$2.43	& 54.80$\pm$0.03	& 50.59$\pm$3.61\\
				
				& \emph{Acc}			& 50.89$\pm$1.74& \underline{83.98$\pm$3.96}	& 72.83$\pm$1.89& 75.33$\pm$2.39	& 81.79$\pm$2.69	& \textbf{80.72$\pm$1.35}\\
				\multirow{-2}{*}{\textbf{DNN}}
				& \emph{F1}				& 35.94$\pm$0.11& \underline{83.69$\pm$4.62}	& 58.73$\pm$0.06& 67.05$\pm$0.89	& 68.44$\pm$0.09	& \underline{78.09$\pm$1.11}\\
				
				& \emph{Acc} & \textbf{98.75$\pm$0.06} & \textbf{85.44$\pm$0.01} & \textbf{97.39$\pm$0.11} & \textbf{84.83$\pm$0.36} & {\underline{83.67$\pm$0.35}}       & {\underline{79.54$\pm$0.15} } \\
				\multirow{-2}{*}{\textbf{Ours}}               
				& \emph{F1}  & \textbf{97.94$\pm$0.46} & \textbf{84.12$\pm$0.18} & \textbf{97.09$\pm$0.04} & \textbf{84.95$\pm$0.29} & \underline{ 84.60$\pm$0.39}		& \textbf{80.14$\pm$0.11}                   \\ 
				\bottomrule
			\end{tabular}
		}
	\end{table*}

\begin{figure*}[]
	\centering
	%		\vspace{-4mm}	
	\subfigure[UMIST: DNN]{
		\label{UMIST: DNN}
		\includegraphics[scale=0.155]{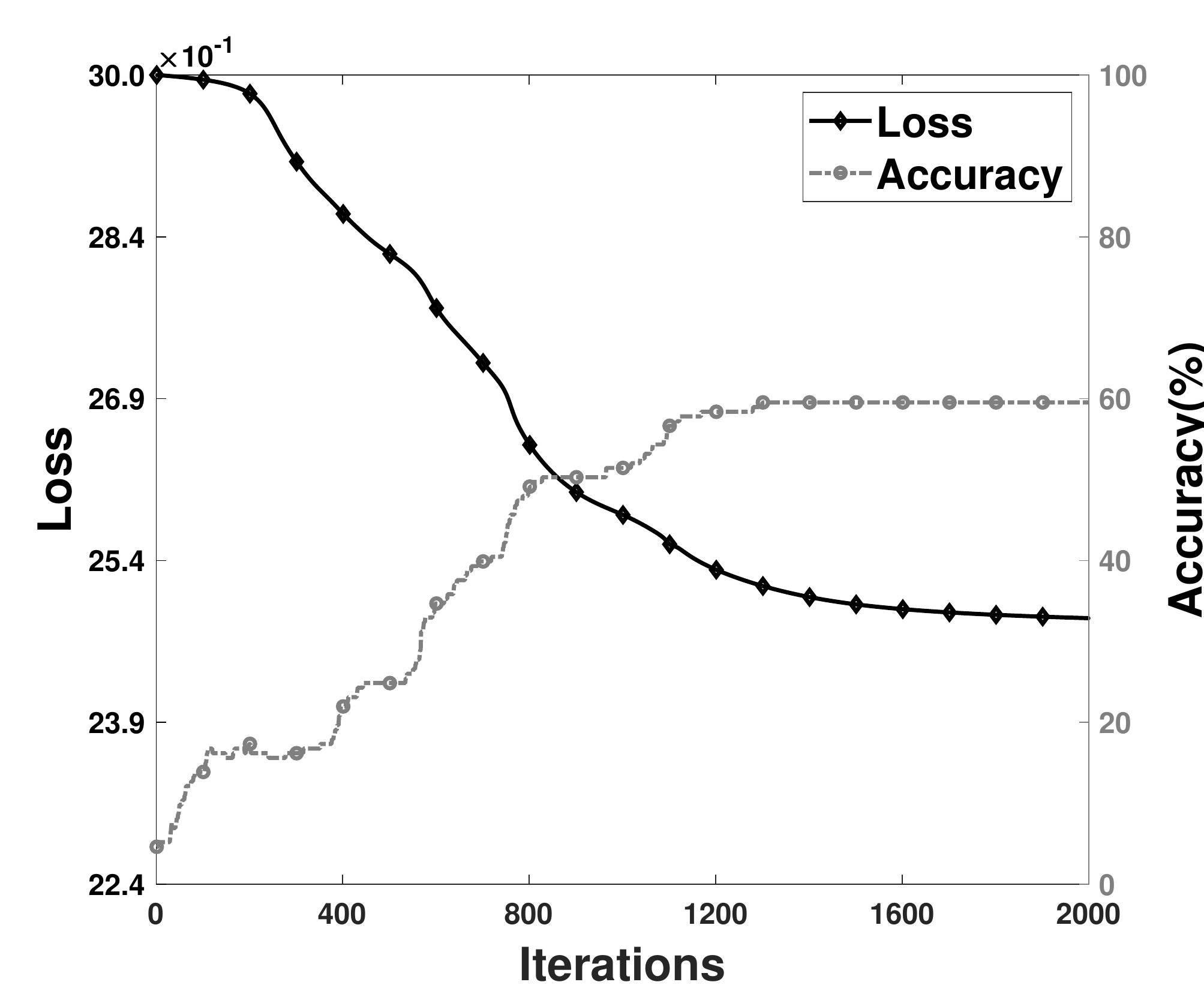}
	}
	\subfigure[UMIST: OURS]{
		\label{UMIST: OURS}
		\includegraphics[scale=0.155]{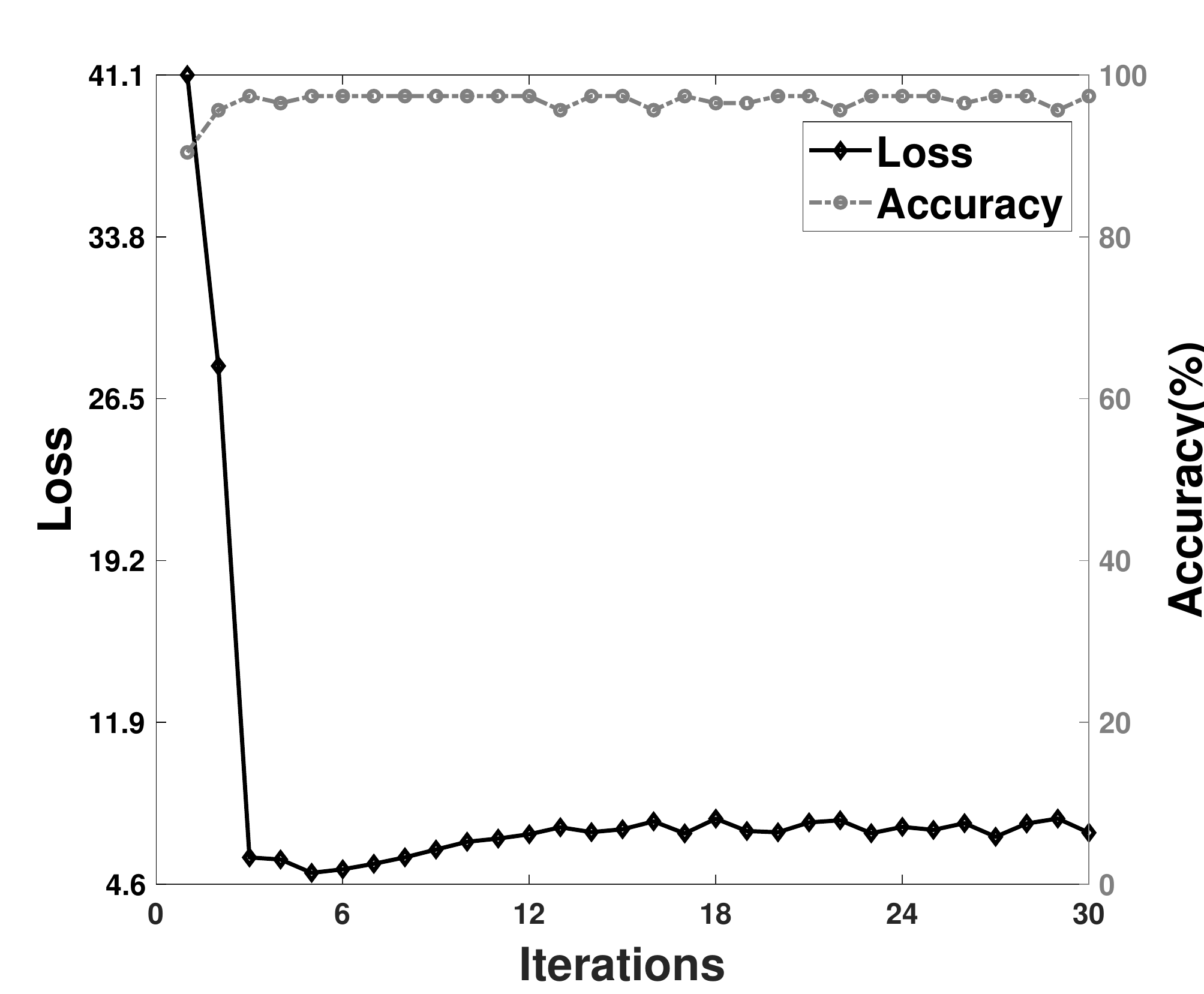}
	}
	\subfigure[MNIST-Mini: DNN]{
		\label{MNIST-Mini: DNN}
		\includegraphics[scale=0.155]{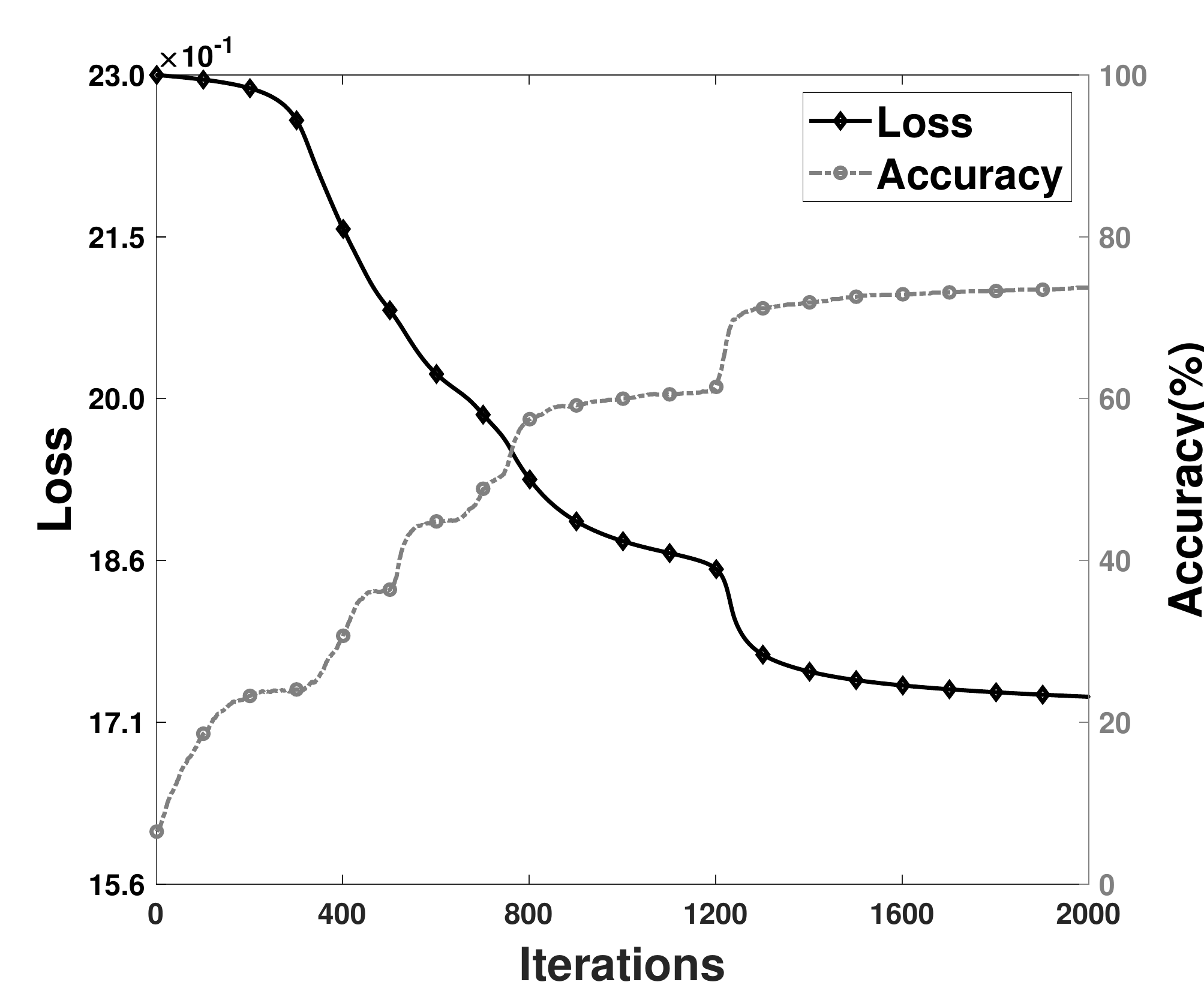}
	}
	\subfigure[MNIST-Mini: OURS]{
		\label{MNIST-Mini: OURS}
		\includegraphics[scale=0.155]{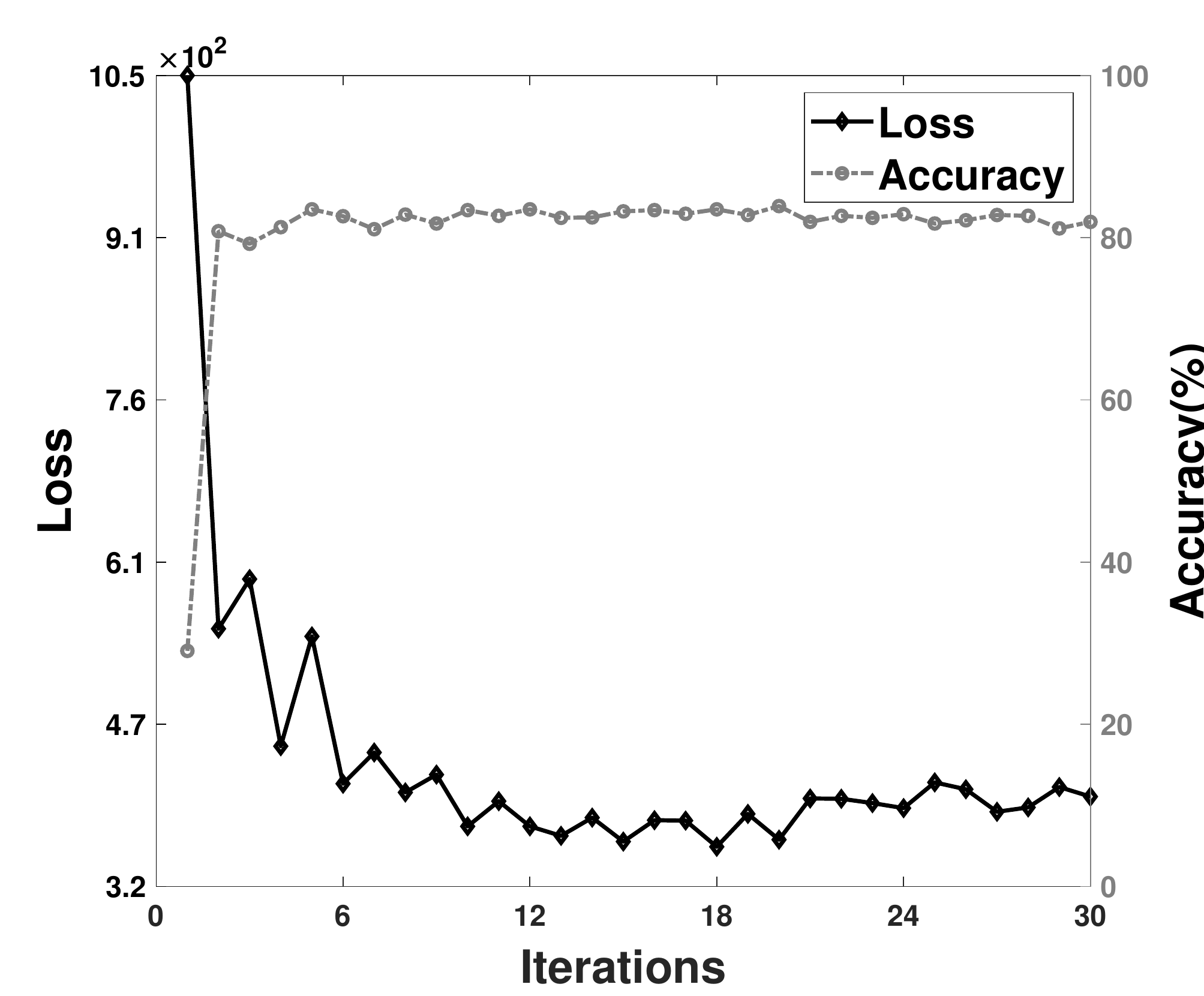}
	}
	\caption{Relationship between loss, accuracy, and iterations for DNN and proposed network on datasets including UMIST, MNIST-Mini. The four figures show the performance and convergence of methods.}
	\label{Convergence comparison}
	%		\vspace{-7mm}	
\end{figure*}

	\subsection{Dataset Description and Experimental Settings} \label{dataset description}

	Six benchmark datasets involved in our experiments are AT\&T\cite{341300}, WAVEFORM\cite{2013UCI}, UMIST\cite{6565365}, MNIST-MINI, MNIST\cite{726791}, and FashionMNIST\cite{xiao2017fashionmnist}. The detail of these datasets is listed in Table \ref{table_datasets}. We compare the proposed network with six classic methods including Ridge Regression ($RidgeReg$), Lasso Regression ($LassoReg$), the Hierarchical Multiclass Support Vector Machine ($SmartSVM$) \cite{2017Fast}, the Multiclass Kernel-based Vector Machine ($MKVM$) \cite{Crammer2002On}, SVM with OVR ($SVM$) and Deep Neural Network ($DNN$).

	Before classification, all benchmark datasets are accordingly (row-)normalized to the range of $[0,1]$. Mnist and FashionMnist are used $60000$ data as the training data and $10000$ data as the testing data. Except for these two datasets, the other datasets are split into a training dataset ($80\%$) and a testing dataset ($20\%$). Because several methods involve trade-off coefficients, we choose values of these parameters from $\left\{ 2^{-3}, 2^{-2}, 2^{-1}, 2^0, 2^1, 2^2, 2^3 \right\}$ and record the best results. The gradient descent is utilized to optimize the DNN. The learning rate is $0.01$ and the iteration is $2000$. The iteration of our model is 30. Apart from that, the proposed model and DNN use the same structure. For AT\&T, two layers are used and the hidden units are 64. The other five datasets adopt 3 layers. Among them, the hidden units in WAVEFORM are 10 and 4. The hidden units in UMIST are 64 and 32. For MNIST and FashionMinst, the hidden units are 32 and 16. Accuracy and F1-score are employed to evaluate these different methods.
	
	\subsection{Analysis of Experiments}
	Aiming to evaluate the performance of the proposed model more accurately, the models mentioned above are run 10 times on the datasets. The results including mean value and standard deviation are shown in Table \ref{table_result}. Based on this, we can conclude that
	
	1) The proposed network achieves great performance on all benchmark datasets. It obtains the top accuracy and F1-score of classification on the first four datasets. On the Mnist, it also achieves the second-highest score. 
	
	2) When the number of categories is large, ours still has a good ability to predict the label accurately especially on AT\&T and UMIST.
	
	3) With the help of a more reasonable and efficient decision layer, the proposed novel network compared with DNN has a better performance.
	
	Furthermore, aiming to study the convergence and performance of the proposed network, we compare the loss and accuracy of DNN and ours with the number of iterations. The results are shown in Fig \ref{Convergence comparison}. The proposed network can converge within 5 iterations regardless of the size of the datasets. On the contrary, it is almost necessary to optimize the DNN by more than 1000 iterations to converge. Based on this, we can infer that the designed non-gradient strategy is more efficient than gradient descent. Besides, by embedded the flexible Stiefel manifold and utilized the adaptive weights,  the proposed network successfully explores the inner distribution of the data and obtains a higher accuracy.

	%\begin{enumerate}
	%	\item [1)] The proposed network achieves the top accuracy and F1-score of classification on the first four datasets. Apart from that, 
	%	\item [2)] 
	%	\item [3)] 
	%	\item [4)] 
	%\end{enumerate}

	% \usepackage{multirow}
	% \usepackage[normalem]{ulem}

	\subsection{Sensitivity Analysis w.r.t. Parameter $\lambda$}
	
	In this part, we conduct the corresponding experiments to study the sensitivity of our network Eq. (\ref{manifold network}) regarding to $\lambda$. The two datasets, AT\&T and UMIST, are preprocessed and split according to Subsection \ref{dataset description}. Besides, $\lambda$ varies from $\left\{ 2^{-3}, 2^{-2}, 2^{-1}, 2^0, 2^1, 2^2, 2^3 \right\}$. The accuracy and F1-score related to it are suggested in Fig \ref{Sensitivity}.
	
	$1)$ The curves of two indices are steady when $\lambda$ is relevant small like $\lambda \leq 2^{-1}$. However, the performance of the proposed model drops rapidly when $\lambda$ is larger.
	
	$2)$ Our network is insensitive to trade-off parameter $\lambda$, when $\lambda < 2^{0}$. Therefore, we can either fine-tune $\lambda$ in $(0, 1]$ or simply set it as a median like $0.5$.
	
	\begin{table*}[]
		\renewcommand\arraystretch{1.3}
		\centering
		\caption{Ablation Study on Three Datasets}
		\label{Ablation Study}
		\scalebox{0.80}{
			\begin{tabular}{ccccccc}
				\toprule
				\multicolumn{1}{l}{}	& \multicolumn{2}{c}{\textbf{AT\&T}}	& \multicolumn{2}{c}{\textbf{WAVEFORM}}	& \multicolumn{2}{c}{\textbf{UMIST}}\\ \cline{2-7} 
				\multicolumn{1}{l}{\multirow{-2}{*}{\textbf{}}} 
				& \emph{ACC}	& \emph{F1}             & \emph{ACC}		& \emph{F1}         & \emph{ACC}		& \emph{F1}       \\ \hline
				\textbf{DNN}            & 50.89$\pm$1.74& 35.94$\pm$0.11        & 83.98$\pm$3.96    & 83.69$\pm$4.62    & 72.83$\pm$1.89    & 58.73$\pm$0.06  \\
				\textbf{RidgeR-NN}& 90.86$\pm$0.03& 91.42$\pm$0.02        & 68.49$\pm$4.11    & 65.94$\pm$5.23    & 93.91$\pm$0.47    & 93.92$\pm$1.19  \\
				\textbf{RidgeR-SVM-NN}  & \underline{94.17$\pm$0.02}& \underline{95.75$\pm$0.01}    & \underline{84.26$\pm$1.80}& \underline{83.85$\pm$1.81}    
				& \underline{97.39$\pm$0.47}                         & { \underline{97.01$\pm$0.47}}    \\
				\textbf{RidgeR-SVM-AW-NN}& \textbf{98.75$\pm$0.06} & \textbf{97.94$\pm$0.46} & \textbf{85.44$\pm$0.01} & \textbf{84.12$\pm$0.18} & \textbf{97.39$\pm$0.11} & \textbf{97.09$\pm$0.04} \\ 
				\bottomrule
			\end{tabular}
		}
		%		\vspace{-5mm}	
	\end{table*}

	\begin{figure*}[]
		\centering
		\subfigure[AT\&T]{
			\label{att40}
			\includegraphics[scale=0.3]{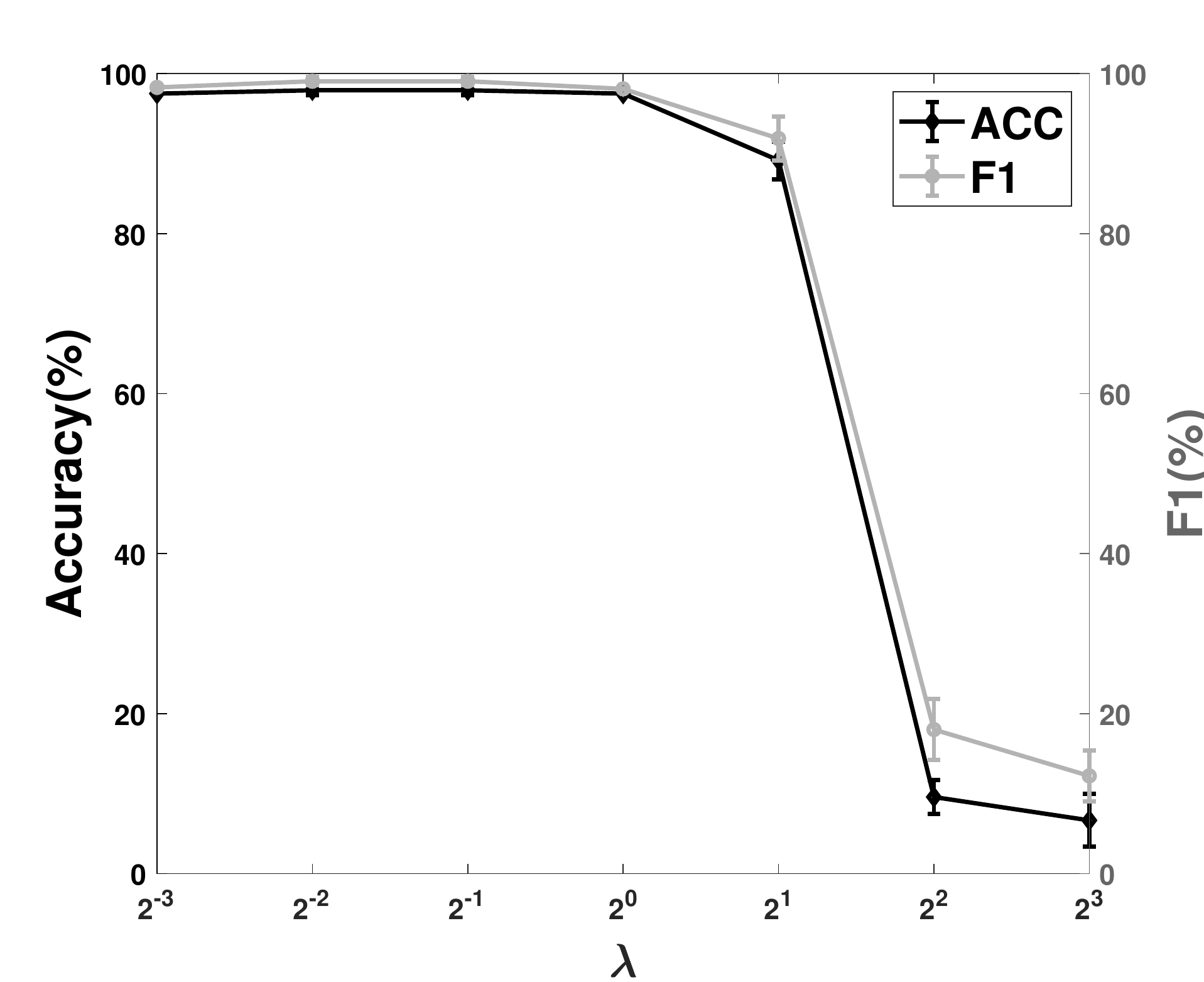}
		}
		\subfigure[UMIST]{
			\label{umist}
			\includegraphics[scale=0.3]{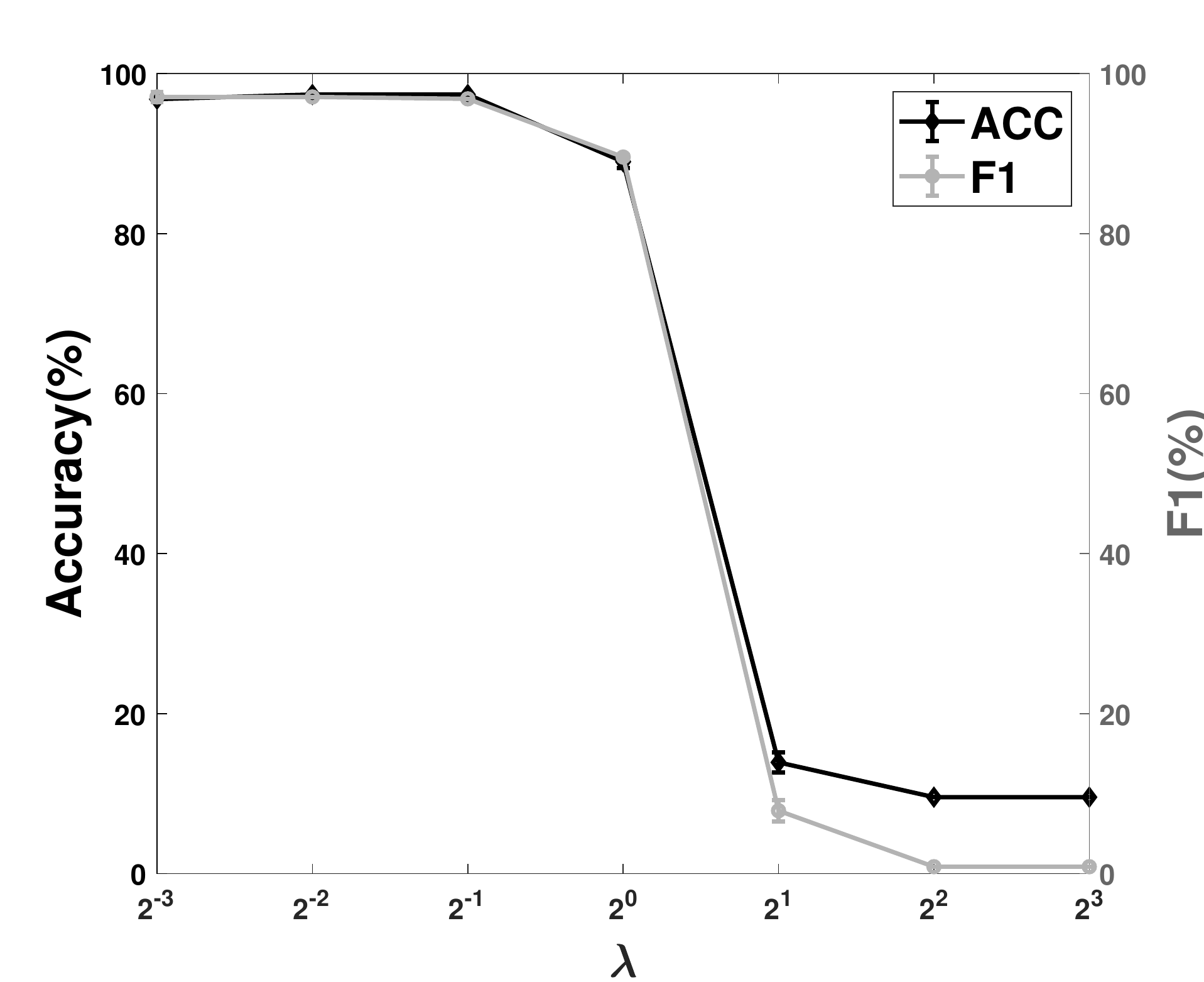}
		}
		%	\subfigure[MNIST-Mini: DNN]{
		%		\label{MNIST-Mini: DNN}
		%		\includegraphics[scale=0.155]{./figures/sensitivity/att40.eps}
		%	}
		%	\subfigure[MNIST-Mini: OURS]{
		%		\label{MNIST-Mini: OURS}
		%		\includegraphics[scale=0.155]{./figures/sensitivity/umist.eps}
		%	}
		\caption{Accuracy and F1-score of the proposed network w.r.t the varying parameter $\lambda \in \left\{ 2^{-3}, 2^{-2}, 2^{-1}, 2^0, 2^1, 2^2, 2^3 \right\}$.}
		\label{Sensitivity}
		\vspace{-5mm}	
	\end{figure*}

	\subsection{Ablation Study}
	
	We conduct an ablation study to evaluate how each part of the proposed network including reconstruction strategy with ridge regression (RidgeR), the flexible Stiefel Manifold SVM (SVM), and the adaptive weight (AW), contribute to overall model performances. Besides, the classic DNN with softmax is utilized as a baseline. The results are summarized in Table \ref{Ablation Study}. We confirm that the non-gradient optimization algorithm \ref{neural network} provides the RidgeR network with the closed-form results, which is helpful to improve the performance. Moreover, contrasted with softmax, SVM assists the model to discover the latent data distribution. Besides, it is shown that unifying the SVM with adaptive weight can make the model more interpretable and further improve performance.

	\section{Conclusion}
	In this paper, we propose a novel non-gradient neural network. Compared with the classic deep neural network optimized via gradient descent, it can be solved with closed-form results via the proposed algorithms which have a faster convergent speed. Besides, it unifies the flexible Stiefel manifold and adaptive weights into the SVM acted as the decision layer, which successfully utilizes the inner structure of the data to predict the label and improve the performance significantly. On the benchmark datasets, the proposed network achieves excellent results. \textbf{In the future, we will attempt to provide the convolutional neural network with a non-gradient optimized strategy}.
	
	\bibliographystyle{nips.bst}
	\bibliography{usvm.bib}
	
	\newpage
	\begin{appendix}
		\setcounter{section}{0}
		\setcounter{myLemma}{0}
		\setcounter{myTheo}{0}
		\section{Proof of Lemma 1}
		\begin{myLemma}
			\label{cal_b}
			Given a feature matrix $\bm X$ and a label matrix $\bm Y$, the problem 
			\begin{equation}
				\label{loss_b}
				\min \limits_{\bm{b}} ||\bm{W}^T\bm{X} + \bm{b}\bm{1}_n^T-\bm{Y}||_F^2+\lambda ||\bm{W}||_F^2
			\end{equation}
			%		$\min \limits_{\bm{b}} ||\bm{W}^T\bm{X} + \bm{b}\bm{1}_n^T-\bm{Y}||_F^2+\lambda ||\bm{W}||_F^2$
			can be solved by $\bm{b}=\frac{\bm{Y}-\bm{W}^T\bm{X}}{n} \bm{1}_n$.
		\end{myLemma}
		
		\begin{proof}
			Lemma \ref{cal_b} aims to solve the the $\bm b$. Therefore, the problem (\ref{loss_b}) is formulated as
			\begin{equation}
				\label{simplify_b}
				\begin{split}
					\mathcal{J}&= \min \limits_{\bm b} ||\bm{W}^T\bm{X} + \bm{b}\bm{1}_n^T-\bm{Y}||_F^2\\
					& \Rightarrow  \min \limits_{\bm b} {\rm Tr}((\bm{W}^T\bm{X} + \bm{b}\bm{1}_n^T-\bm{Y})^T(\bm{W}^T\bm{X} + \bm{b}\bm{1}_n^T-\bm{Y})) \\
					& \Rightarrow  \min \limits_{\bm b} 2{\rm Tr}(\bm{1}_n \bm{b}^T \bm{W}^T \bm{X}) +
					{\rm Tr}(\bm{1}_n \bm{b}^T \bm{b} \bm{1}_n^T) + 2{\rm Tr}(\bm{1}_n \bm{b}^T \bm{Y}).
				\end{split}
			\end{equation}
			Owing to constraint on $\bm b$, Eq. (\ref{simplify_b}) is derivated w.r.t $\bm b$ and set to $0$ like
			\begin{equation}
				\label{derivate_b}
				\left\{
				\begin{split}
					\nabla_{\bm{b}} \mathcal{J}&=0\\
					\bm{b} \bm{1}_n^T \bm{1}_n &= \bm{Y}\bm{1}_n -\bm{W}^T \bm{X} \bm{1}_n. \\
				\end{split}
				\right.
			\end{equation}
			Then, we have  $\bm{b}=\frac{\bm{Y}-\bm{W}^T\bm{X}}{n} \bm{1}_n$.
		\end{proof}
		
		\section{Proof of Lemma 2}
		\begin{myLemma} 
			\label{cal_w}
			Given a feature matrix $\bm X$ and a label matrix $\bm Y$, the problem 
			\begin{equation}
				\label{loss_w}
				\min \limits_{\bm{W}} ||\bm{W}^T\bm{X}-\bm{Y}||_F^2+\lambda ||\bm{W}||_F^2
			\end{equation}
			can be solved by $\bm{W}= (\bm{X}\bm{X}^T+\lambda\bm{I}_d)^{-1} \bm{X}\bm{Y}^T$.
		\end{myLemma}
		
		\begin{proof}
			Lemma \ref{cal_w} aims to solve the the $\bm W$. Therefore, the problem (\ref{loss_w}) is formulated as
			\begin{equation}
				\label{simplify_w}
				\begin{split}
					\mathcal{L}&=\min \limits_{\bm{W}} ||\bm{W}^T\bm{X}-\bm{Y}||_F^2+\lambda ||\bm{W}||_F^2 \\
					& \Rightarrow \min \limits_{\bm{W}} {\rm Tr}((\bm{W}^T\bm{X}-\bm{Y})^T(\bm{W}^T\bm{X}-\bm{Y})) + \lambda{\rm Tr}(\bm{W}^T \bm{W}) \\
					& \Rightarrow \min \limits_{\bm{W}} {\rm Tr}(\bm{X}^T \bm{W} \bm{W}^T \bm{X}) - 2{\rm Tr}(\bm{W}^T \bm{X} \bm{Y}^T) + \lambda {\rm Tr}(\bm{W}^T \bm{W}).
				\end{split}
			\end{equation}
			Then, Eq. (\ref{simplify_w}) can be solved via taking the derivation of $\bm W$ like
			\begin{equation}
				\label{derivate_w}
				\left\{
				\begin{split}
					\nabla_{\bm{W}} \mathcal{L}&=0\\
					(\bm{X} \bm{Y}^T + \lambda \bm{I}_d) \bm{W} &= \bm{X}\bm{Y}^T. \\
				\end{split}
				\right.
			\end{equation}
			Due to that $\bm{X} \bm{Y}^T + \lambda \bm{I}_d$ is a full rank matrix, we can obtain the $\bm{W}= (\bm{X}\bm{X}^T+\lambda\bm{I}_d)^{-1} \bm{X}\bm{Y}^T$.
		\end{proof}

		\section{Proof of Theorem 1}
		\begin{myTheo} 
			\label{cal_alpha}
			Suppose that $f_{(1)} \leq f_{(2)} \leq ... \leq f_{(n)}$. If $\gamma=\frac{n-1}{2}f_{(n)}-\frac{1}{2} \sum_{i=1}^{n-1}f_{(i)}$, the optimal $\bm{\alpha}$ is 
			\begin{equation}
				\label{cal_a}
				\alpha_i = (\frac{f_{(n)} - f_{(n-1)}}{(n-1) f_{(n)} - \sum \limits_{j=1}^k f_{(j)}})_+ .
			\end{equation}
		\end{myTheo}
		
		\begin{proof}
			Theorem \ref{cal_alpha} aims to solve the $\bm{\alpha}$. Therefore, we fix $M$, $W$ and $\bm b$. The objective function is formulated as 
			\begin{equation}
				\label{obj_alpha}
				\min \limits_{\textbf{1}_n^T \bm \alpha = 1, \bm \alpha \geq 0} \sum \limits_{i=1}^n \alpha_i f_i + \gamma \alpha_i^2, 
			\end{equation}
			where $f_i = \| W^T \bm x_i + \bm b - \bm y_i - \bm y_i \odot \bm m_i \|_2^2$, $\gamma$ is the trade-off coefficient, and $\bm \alpha \in R^d$ is the weight vector. The first term in problem.(\ref{obj_alpha}) means that a points with large classification errors should be assigned with a small weight. In general, $f_i$ can be viewed as a constant value and Eq. (\ref{obj_alpha}) has the same solutions with following problem
			\begin{equation}
				\min \limits_{\textbf{1}_n^T \bm \alpha = 1, \bm \alpha \geq 0} \|\bm \alpha + \frac{\bm f}{2 \gamma}\|_2^2,
			\end{equation}
			where $\bm f = [f_1,f_2,...,f_n]$.
			Transformed with the Lagrangian function, the current problem is represented as
			\begin{equation}
				\mathcal{L}_\alpha  = \|\bm \alpha + \frac{\bm f}{2 \gamma}\|_2^2 + \xi (1 - \sum \limits_{i=1}^n \alpha_i) - \sum \limits_{i=1}^n \beta_i \alpha_i,
			\end{equation}
			where $\xi$ and $\beta_i$ are the Largranian multipliers. The KKT conditions are given as
			\begin{equation}
				\left\{
				\begin{array}{l}
					\frac{\partial \mathcal{L}_\alpha}{\partial \alpha_i} = \alpha_i + \frac{f_i}{2 \gamma} - \xi - \beta_i = 0 \\
					\beta_i \alpha_i = 0 \\
					\sum \limits_{i=1}^n \alpha_i = 1, \beta_i \geq 0, \alpha_i \geq 0
				\end{array}
				.
				\right.
			\end{equation}
			Consider the following cases
			\begin{equation}
				\left\{
				\begin{array}{l}
					\alpha_i = 0 \Rightarrow \xi - \frac{f_i}{2 \gamma} = - \beta_i \leq 0 \\
					\alpha_i \geq 0 \Rightarrow \alpha_i = \xi - \frac{f_i}{2 \gamma}
				\end{array}
				,
				\right.
			\end{equation}
			which means
			\begin{equation}
				\alpha_i = (\xi - \frac{f_i}{2 \gamma})_+.
			\end{equation}
			Without loss of generality, we suppose that $f_{(1)} \leq f_{(2)} \leq ... \leq f_{(n)}$ and $\alpha_1 \geq \alpha_2 \geq \cdots \geq \alpha_n$. Therefore, we have
			\begin{equation}
				\left\{
				\begin{array}{l}
					\alpha_k \textgreater 0 \Rightarrow \xi - \frac{f_{(k)}}{2 \gamma} > 0\\
					\alpha_{k+1} = 0 \Rightarrow \xi - \frac{f_{(k+1)}}{2 \gamma} \leq 0
				\end{array}
				,
				\right.
			\end{equation}
			where $k \in [1, n]$. Due to $\sum \limits_{i=1}^n \alpha_i = 1$, we have
			\begin{equation}
				\sum \limits_{i=1}^n \alpha_i = k \xi - \sum \limits_{i=1}^k \frac{f_{(i)}}{2 \gamma} = 1,
			\end{equation}
			which means
			\begin{equation}
				\xi = \frac{1}{2 k \gamma} \sum \limits_{i=1}^k f_{(i)} + \frac{1}{k}.
			\end{equation}
			Combine with our assumption and we have
			\begin{equation}
				\begin{array}{l}
					\frac{1}{2 k \gamma} \sum \limits_{i=1}^k f_{(i)} + \frac{1}{k} - \frac{f_{(k+1)}}{2 \gamma} \leq  0 <  \frac{1}{2 k \gamma} \sum \limits_{i=1}^k f_{(i)} + \frac{1}{k} - \frac{f_{(k)}}{2 \gamma} \\
					\Rightarrow  \frac{f_{(k)}}{2 \gamma} <  \frac{1}{2 k \gamma} \sum \limits_{i=1}^k f_{(i)} + \frac{1}{k}  \leq \frac{f_{(k+1)}}{2 \gamma} \\
					\Rightarrow \frac{k f_{(k)}}{2} - \frac{1}{2} \sum \limits_{i=1}^k f_{(i)} < \gamma \leq \frac{k f_{(k+1)}}{2} - \frac{1}{2} \sum \limits_{i=1}^k f_{(i)}
					.
				\end{array}
			\end{equation}
			When $k=n-1$ and $\gamma=\frac{n-1}{2}f_{(n)}-\frac{1}{2} \sum_{i=1}^{n-1}f_{(i)}$, we obtain
			\begin{equation}
				\label{update_alpha}
				\begin{split}
					\alpha_i & = (\frac{\sum \limits_{i=1}^{n-1} f_{(i)} + 2\gamma}{2 (n-1) \gamma} - \frac{f_{(i)}}{2 \gamma})_+  = (\frac{f_{(n)} - f_{(n-1)}}{(n-1) f_{(n)} - \sum \limits_{i=1}^{n-1} f_{(i)}})_+
				\end{split}
				.
			\end{equation}
			
		\end{proof}
		
		\section{Proof of Theorem 2}
		\begin{myTheo}
			\label{solve manifold}
			Suppose that $\bm{\alpha}$ is a constant and let $\bm{G}=\bm{Y}-\bm{Y}\odot \bm{M}$. $\bm W$ and $\bm b$ can be solved by
			\begin{equation}
				\label{W_b}
				\bm{W} = \bm{S}^{-1} \bm{U} \bm{\Lambda}^T \bm{V}^T \quad and \quad
				\bm{b} = \frac{\bm{G}\bm{D}\bm{1}_n-\bm{W}^T\bm{X}\bm{D}\bm{1}_n}{\bm{1}_n^T\bm{D}\bm{1}_n}
				%			\min \limits_{\bm U} ||\bm{H}^{(m)}_{n_l} \bm{U} - \bm{Y}_{n_l}||_F^2 \\
				%		\left\{
				%		\begin{array}{l}
				%			\bm{W} = \bm{S}^{-1} \bm{U} \bm{\Lambda}^T \bm{V}^T \\
				%			\bm{b} = \frac{\bm{G}\bm{D}\bm{1}_n-\bm{W}^T\bm{X}\bm{D}\bm{1}_n}{\bm{1}_n^T\bm{D}\bm{1}_n}\\
				%			%			\min \limits_{\bm U} ||\bm{H}^{(m)}_{n_l} \bm{U} - \bm{Y}_{n_l}||_F^2 \\
				%		\end{array}
				%		\right.
				%		,
			\end{equation}
			where $\bm{S}=(\bm{X}\hat{\bm{D}}\bm{X}^T+\lambda \bm{I}_d)^{\frac{1}{2}}$, $\hat{\bm{D}}=\hat{\bm{C}}\bm{D}\hat{\bm{C}}^T$,
			$\bm{D}={\rm diag}(\bm{\alpha})$ and $\bm{\Lambda}=[\bm{I}_c, \bm{0}]\in R^{c\times d}$. $\bm{U},\bm{V}$ are the left and right singular of $\bm{S}^{-1}\bm{X}\hat{\bm{D}}\bm{G}$, respectively. Having obtained the solution of $\bm W, \bm b$, slack variable $\bm M$ can be solved by
			\begin{equation}
				\label{cal_M}
				\bm{m}_i=(\bm{y}_i \odot (\bm{W}^T \bm{x}_i+\bm{b})-\bm{1}_c)_+.
			\end{equation}
		\end{myTheo}
		
		\begin{proof}
			Theorem \ref{solve manifold} aims to obtain the analytic solutions of $\bm{M}$, $\bm W$ and $\bm b$. Therefore, the objective function is formulated as
			\begin{equation}
				\label{loss_w_b_m}
				\min \limits_{\bm{M} \geq 0,\bm{W},\bm{b}} \sum \limits_{i=1}^n \alpha_i f_i +\lambda \|\bm W\|_F^2, \  
				s.t. \ \bm{W}^T(\bm{X}\hat{\bm{C}}\bm{D}\hat{\bm{C}}^T\bm{X}^T+\lambda \bm{I}_d)\bm{W}=\bm{I}_c,\\
			\end{equation} 
			where $f_i = \| W^T \bm x_i + \bm b - \bm y_i - \bm y_i \odot \bm m_i \|_2^2$. These variables can be divided into two parts, solving slack variable $\bm M$ and optimizing parameters, $\bm W$ and $\bm b$.
			
			$A.\; Optimize \; \bm{M} \; with \; fixing \; \bm{W} \; and \; \bm{b}$
			
			Regarding $\bm W$, $\bm b$ and $\bm \alpha$ as constants, Eq. (\ref{loss_w_b_m}) to optimize the $\bm M$ is formulated as the following sub-problems 
			\begin{equation}
				\label{m1}
				\begin{split}
					\min \limits_{\bm m_i \geq 0} \|W^T \bm x_i + b - \bm y_i - \bm y_i \odot \bm m_i \|_2^2.
					%				\Leftrightarrow &\min \limits_{\bm m_i \geq 0} \|\bm y_i \odot (W^T \bm x_i + b) - \textbf{1}_c - \bm m_i \|_2^2 , \\
				\end{split}
			\end{equation}
			Due to $\bm y_i \odot \bm y_i = 1$, Eq. (\ref{m1}) is defined like
			\begin{equation}
				\label{m2}
				\min \limits_{\bm m_i \geq 0} \|\bm y_i \odot (W^T \bm x_i + b) - \textbf{1}_c - \bm m_i \|_2^2.
			\end{equation}
			We utilize the Lagrange multipliers to transform Eq. (\ref{m2}) like
			\begin{equation}
				\mathcal{L}_m = \|\bm y_i \odot (W^T \bm x_i + \bm b) - \textbf{1}_c - \bm m_i\|_2^2 - \bm \rho_i^T \bm m_i ,
			\end{equation}
			where $\bm \rho_i \geq 0$ is a Lagrange multiplier. Based on KKT conditions, we have the following derivation
			\begin{equation} \notag
				\left\{
				\begin{array}{l}
					2 \bm m_i - 2 \bm y_i \odot (W^T \bm x_i + \bm b) + 2 - \bm \rho_i = 0 \\
					\bm \rho_i \odot \bm m_i = 0 \\
					\bm \rho_i \geq 0, \bm m_i \geq 0 
				\end{array}
				\right.
				.
				%			\Rightarrow
				%			\left\{
				%			\begin{array}{c l}
				%				m_{i}^{(j)} = 0, & y_i^{(j)} (W^T \bm x_i + \bm b)^{(j)} - 1 = - \frac{\eta_i^{(j)}}{2}\\
				%				m_{i}^{(j)} \geq 0, & m_i = y_i^{(j)} (W^T \bm x_i + \bm b)^{(j)} - 1 .
				%			\end{array}
				%			\right.
			\end{equation}
			For $m_{ij}=0$, $y_{ij} (W^T \bm x_i + \bm b)_j - 1=- \frac{\rho_{ij}}{2}$. Due to $\bm \rho_i \geq 0$, $- \frac{\rho_{ij}}{2}$ is less than or equal to 0. If $m_{ij} \geq 0$, we can obtain that $m_{ij} = y_{ij} (W^T \bm x_i + \bm b)_j - 1$. Therefore, the solution to $\bm{m}_i$ is unified as
			\begin{equation}
				\label{cal_m}
				\bm m_i = (\bm y_i \odot (W^T \bm x_i + \bm b) - \textbf{1}_c)_+ .
			\end{equation}
			
			$B.\; Optimize \; \bm{W} \; and \; \bm{b} \; with \; fixing \; \bm{M}$
			
			Having obtained the $\bm M$, the objective function can be reformulated as 
			\begin{equation}
				\label{w_b_1}
				\mathcal{J}=\min \limits_{\bm{W}, \bm{b}} \sum \limits_{i=1}^n \alpha_i || W^T \bm x_i + \bm b - \bm{g}_i||_F^2 +\lambda \|\bm W\|_F^2, \  
				s.t. \ \bm{W}^T(\bm{X}\hat{\bm{C}}\bm{D}\hat{\bm{C}}^T\bm{X}^T+\lambda \bm{I}_d)\bm{W}=\bm{I}_c,\\
			\end{equation}
			where $\bm G = \bm Y - \bm Y \odot \bm M$. Since there is no constraint on $\bm b$, it can be solved via $\frac{\partial \mathcal{J}}{\partial \bm{b}}=0$ and the closed-form results like
			\begin{equation}
				\label{b_1}
				\bm{b} = \frac{\bm{G}\bm{D}\bm{1}_n-\bm{W}^T\bm{X}\bm{D}\bm{1}_n}{\bm{1}_n^T\bm{D}\bm{1}_n}.
			\end{equation}
			
			Therefore, the bias term is rewritten as
			\begin{equation}
				\label{b_2}
				\bm{b}\bm{1}_n^T=\bm{G}\frac{1}{\bm{1}_n^T \bm{D} \bm{1}_n} \bm{D} \bm{1}_n \bm{1}_n^T-\bm{W}^T \bm{X} \frac{1}{\bm{1}_n^T \bm{D} \bm{1}_n} \bm{D} \bm{1}_n \bm{1}_n^T
				=(\bm{G}-\bm{W}^T\bm{X})(\bm{I}_n-\hat{\bm{C}}).
			\end{equation}
			
			Furthermore, the residual error can be reformulated as
			\begin{equation}
				\label{residual error}
				\bm{W}^T\bm{X}+\bm{b}\bm{1}_n^T-\bm{G}=(\bm{W}^T\bm{X}-\bm{G})-(\bm{W}^T\bm{X}-\bm{G})(\bm{I}_n-\hat{\bm{C}})=(\bm{W}^T\bm{X}-\bm{G})\hat{\bm{C}}.
			\end{equation}
			
			For simplicity, define $\hat{\bm{D}}=\hat{\bm{C}}\bm{D}\hat{\bm{C}}^T$. The problem Eq. (\ref{w_b_1}) can be reformulated as 
			\begin{equation}
				\label{w_b_2}
				\begin{split}
					\mathcal{J}&=\min \limits_{\bm{W}} {\rm Tr} ((\bm{W}^T \bm{X}+\bm{b}\bm{1}_n^T-\bm{G})^T\bm{D} (\bm{W}^T \bm{X}+\bm{b}\bm{1}_n^T-\bm{G}) )+\lambda {\rm Tr} (\bm{W}^T \bm{W}) \\
					& \Rightarrow \min \limits_{\bm{W}} {\rm Tr} (((\bm{W}^T\bm{X}-\bm{G})\hat{\bm{C}})^T \bm{D} ((\bm{W}^T\bm{X}-\bm{G})\hat{\bm{C}}) + \lambda{\rm Tr}(\bm{W}^T \bm{W}) \\
					& \Rightarrow \min \limits_{\bm{W}} {\rm Tr}((\bm{W}^T\bm{X}-\bm{G}) \hat{\bm{D}} (\bm{W}^T\bm{X}-\bm{G})^T) + \lambda {\rm Tr}(\bm{W}^T \bm{W}) \\
					& \Rightarrow \min \limits_{\bm{W}} {\rm Tr}(\bm{W}^T(\bm{X}\hat{\bm{D}}\bm{X}^T+\lambda\bm{I}_d)\bm{W}) -2{\rm Tr}(\bm{W}^T\bm{X}\hat{\bm{D}}\bm{G}^T).
				\end{split}
				%			\mathcal{J}=\min \limits_{\bm{W}, \bm{b}} \sum \limits_{i=1}^n \alpha_i || W^T \bm x_i + \bm b - \bm{g}_i||_F^2 +\lambda \|\bm W\|_F^2, \  
				%			s.t. \ \bm{W}^T(\bm{X}\hat{\bm{C}}\bm{D}\hat{\bm{C}}^T\bm{X}^T+\lambda \bm{I}_d)\bm{W}=\bm{I}_c,\\
			\end{equation}
			Due to that $\bm{W}$ obeys the flexible Stiefel manifold. Therefore, Eq. (\ref{w_b_2}) can be simplified as
			\begin{equation}
				\label{w_loss}
				\mathcal{J}=\max_{\bm{W}^T\bm{S}^2\bm{W}= I} {\rm Tr}(\bm{W}^T\bm{X}\hat{\bm{D}}\bm{G}^T) \Leftrightarrow \mathcal{J}=\max_{\bm{Q}^T \bm{Q}= I} {\rm Tr}(\bm{Q}^T \bm{S}^{-1} \bm{X}\hat{\bm{D}}\bm{G}^T),
			\end{equation}
			where $\bm{S}=(\bm{X}\hat{\bm{D}}\bm{X}^T+\lambda \bm{I}_d)^{\frac{1}{2}}$ and $\bm{Q}=\bm{S}\bm{W}$. According to Lemma \ref{cal_w_manifold},  Eq. (\ref{w_loss}) has the analytic solution like
			\begin{equation}
				\label{result_w}
				\bm{W} = \bm{S}^{-1} \bm{U} \bm{\Lambda}^T \bm{V}^T,
			\end{equation}
			where $[\bm{U},\sim,\bm{V}]={\rm svd}(\bm{S}^{-1}\bm{X}\hat{\bm{D}}\bm{G}^T)$ and $\bm{\Lambda}=[\bm{I}_c, \bm{0}]$.
			
			\begin{myLemma} 
				\label{cal_w_manifold}
				For $\bm{Q},\bm{P} \in R^{m \times n}$ where $m > n$, the problem $\max_{\bm{Q}^T \bm{Q} = I} \bm{Q}^T\bm{P}$ has the closed-form results like $\bm{Q}=\bm{U}\bm{\Lambda}^T\bm{V}^T$. Among them, $[\bm{U},\bm{\Sigma},\bm{V}^T]={\rm svd}(\bm{P})$ and $\bm{\Lambda}=[\bm{I}_n,0]$.
			\end{myLemma}
		\end{proof}
		
		\section{Proof of Lemma 3}
		\begin{proof}
			Note that
			\begin{equation} \notag
				\begin{split}		
					{\rm tr}(\bm{Q}^T \bm{P}) & = {\rm tr}(\bm{Q}^T \bm{U} \bm{\Sigma} \bm{V}^T) = {\rm tr}(\bm{V}^T \bm{Q}^T \bm{U} \bm{\Sigma}) \\
					& = {\rm tr}(\bm{A} \bm{\Sigma}) = \sum \limits_{i = 1}^{n} a_{ii} \sigma_{ii}
				\end{split}	
				.
			\end{equation}
			where $\bm{A} = \bm{V}^T \bm{Q}^T \bm{U}$. Clearly, $\bm{A} \bm{A}^T = \bm{I}$ such that $a_{ij} \leq 1$. Hence, we have
			\begin{equation} \notag
				{\rm tr}(\bm{Q}^T \bm{P}) \leq \sum \limits_{i=1}^n \sigma_{ii}.	
			\end{equation}
			We can simply set $\bm{A} = \bm{V}^T \bm{Q}^T \bm{U} = \bm{\Lambda}$ where $\bm{\Lambda} = [\bm{I}_n, 0]$. In other words,
			\begin{equation}
				\bm{Q} = \bm{U} \bm{\Lambda}^T \bm{V}^T.
			\end{equation}
			Consequently, the lemma is proved.
		\end{proof}
	
	\end{appendix}

\end{document}